\documentclass[11pt]{article}

\usepackage[preprint]{acl}

\usepackage{times}
\usepackage{latexsym}

\usepackage[T1]{fontenc}
\usepackage[utf8]{inputenc}
\usepackage{microtype}
\usepackage{inconsolata}

\usepackage{graphicx}
\usepackage{hyperref}       
\usepackage{url}            
\usepackage{booktabs}       
\usepackage{amsfonts}       
\usepackage{nicefrac}       
\usepackage{microtype}      
\usepackage{xcolor}         
\usepackage{algorithm}
\usepackage[noend]{algpseudocode}
\usepackage{caption}
\usepackage{subcaption}
\usepackage{xcolor}
\usepackage{multirow}
\usepackage{adjustbox}

\usepackage[sectionbib]{chapterbib}
\usepackage{wrapfig}
\usepackage{amsmath}
\usepackage{amssymb}

\usepackage{enumitem}

\usepackage{amsthm}
\newtheorem{proposition}{Proposition}

\newcommand{\method}{AnLR-LoRA}

\title{One Rate Is Not Enough: \\
Adaptive Anisotropic Learning Rates for LoRA Fine-Tuning}
\author{Huiyi Wang\thanks{Corresponding authors.}, Daijiao Liu, Lina Yao, Dong Gong\footnotemark[1]\\
  University of New South Wales
\\
\texttt{ \small \{huiyi.wang,~daijiao.liu,~lina.yao,~dong.gong\}@unsw.edu.au}}

\begin{document}
\maketitle
\begin{abstract}
Low-rank adaptation (LoRA) has become the standard for parameter-efficient fine-tuning of large language models. Most LoRA variants follow a \emph{uniform-LR convention}, applying a single global learning rate across every rank-one component of every adapter. We show that this convention overlooks substantial within-module heterogeneity, where the rank-one components of a LoRA adapter update at highly uneven rates and low-velocity modules converge to concentrated singular spectra that underutilize the nominal rank budget. To address this, we propose an adaptive anisotropic learning-rate model that assigns each rank-one component its own effective learning rate, computed online from training-time signals and mean-normalized per module to preserve the global LR budget. AnLR-LoRA instantiates this model with two signals available during AdamW optimization, namely function-space velocity and Adam SNR, as a lightweight scheme with no extra trainable parameters. Across commonsense reasoning, natural language generation and visual instruction-tuning benchmarks, AnLR-LoRA consistently improves over LoRA while encouraging broader use of rank capacity, with gains that remain robust across a wide range of global learning rates and transfer cleanly to other LoRA variants.
\end{abstract}

\section{Introduction}

Large language models (LLMs) have become the foundation of modern natural language processing, but adapting them to downstream tasks via full 
\begin{figure}[htp]
    \centering
    \begin{subfigure}{0.48\linewidth}
        \centering
        \includegraphics[width=\linewidth]{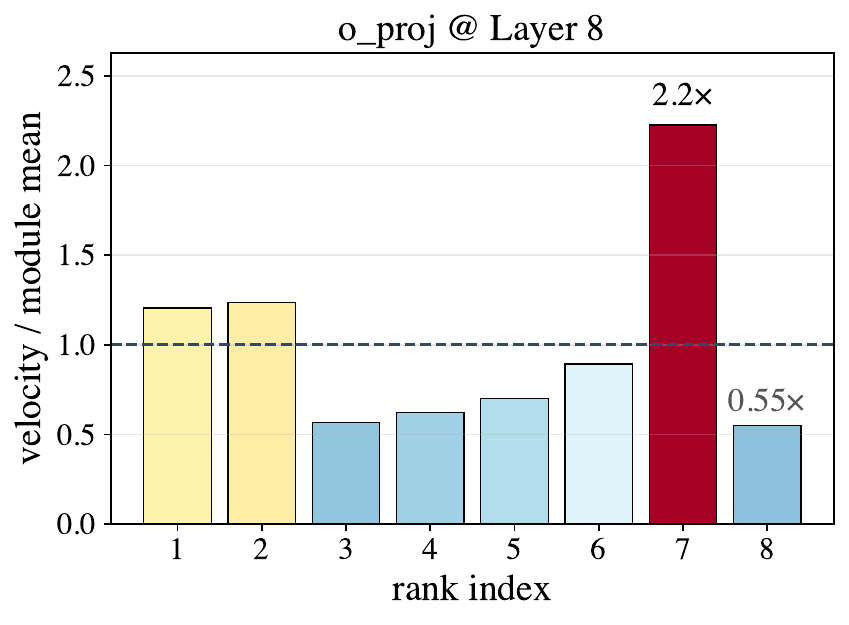}
    \end{subfigure}
    \hfill
    \begin{subfigure}{0.48\linewidth}
        \centering
        \includegraphics[width=\linewidth]{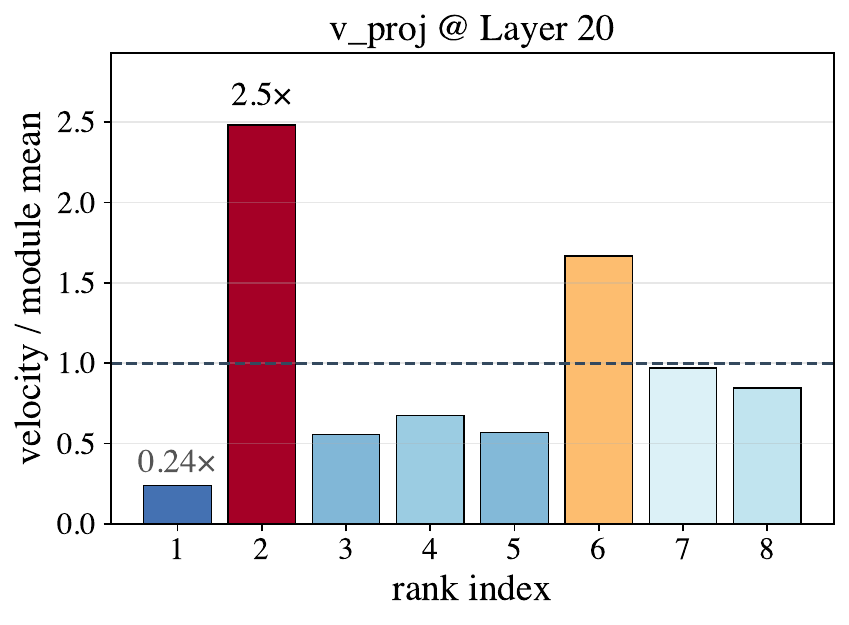}
    \end{subfigure}
    \caption{Per-rank-one update velocity of 
    LoRA modules. Within a single LoRA, the $r$ rank-one components do not update at the same rate.}
    \label{fig:fig_bars}
\end{figure}
fine-tuning is prohibitively expensive. Parameter-efficient fine-tuning (PEFT) methods address this cost by training only a small subset of parameters \citep{houlsby2019parameter, li2021prefix, jia2022visual}. Low-Rank Adaptation (LoRA) \citep{hu2021lora} has emerged as a standard choice for LLM adaptation, with its low-rank update parameterization offering a favorable balance between adaptation quality and computational cost. Yet recent work identifies the learning rate (LR) as the dominant factor in LoRA tuning~\citep{lee2026learning}, which makes the \emph{form} of the learning-rate convention, not just its value, a worthwhile object of study.

Standard LoRA training uses AdamW with a single global learning rate applied uniformly across every rank-one component of every adapter. We refer to this design as the \emph{uniform-LR convention}, which is inherited by most LoRA variants. Our analysis (Sec.~\ref{sec:motivation}), however, shows that under this convention the $r$ rank-one components within a single adapter exhibit anisotropic per-step velocities (Fig.~\ref{fig:fig_bars}). This anisotropy emerges early in training and persists throughout. This imbalance compounds over training, as modules whose rank-one components evolve slowly converge to a concentrated spectrum of the LoRA update $BA$, leaving much of the nominal rank budget underutilized. 
The anisotropy is \emph{within-module} and arises only during training, and it cannot be corrected by tuning a single global LR. This suggests that both the granularity and adaptivity of LoRA's LR allocation deserve closer examination.

This anisotropy in LoRA training motivates assigning learning rates at the level of individual rank-one components and adjusting them as training proceeds. However, existing approaches still allocate learning rates at coarser granularities. LoRA+~\citep{loraplus} partially relaxes the uniform-LR convention by assigning separate learning rates to matrices $A$ and $B$, but the adjustment remains at the matrix level and is fixed analytically before training. Beyond LoRA+, per-matrix and per-adapter schemes~\citep{chen2026learning, huang2024allora} adjust the learning rate above the rank-one level, while methods that operate at the rank-one level intervene through gating~\citep{sora} or singular-triplet pruning~\citep{adalora} rather than learning-rate scheduling. No existing method assigns a distinct, online adaptive learning rate to each rank-one component within a LoRA adapter.

We propose an \emph{adaptive anisotropic learning-rate} model for LoRA fine-tuning, with a principled instantiation, Anisotropic Learning-Rate LoRA (AnLR-LoRA). AnLR-LoRA assigns per-rank-one multipliers that redistribute a fixed global LR across the $r$ rank-one components of each adapter, driven online by two training-time signals: \emph{velocity}, the rate of change of each rank-one contribution to the adapter update, which captures function-space activity; and \emph{Adam SNR}, the signal-to-noise ratio of each component's parameters, which captures parameter-space confidence. The two signals are fused log-additively and mean-normalized per module, then applied as a post-step scaling of the AdamW update, equivalent to running AdamW with an effective per-component LR (Proposition~\ref{prop:post-step}). These signals adaptively encourage learning on low-velocity rank-one components that follow a reliable update direction. AnLR-LoRA thus promotes effective learning with no architectural changes and negligible computational overhead.

Our contributions are summarized as follows:
\begin{itemize}
    \item We identify \emph{anisotropic} learning rates, with distinct rates per rank-one component, as the natural next granularity in the adaptive-LR lineage and show per-rank-one dynamics are anisotropic within modules, persistent through training and tied to a measurable spectral concentration of $BA$ below its nominal rank.
    \item We propose an adaptive anisotropic LR model that redistributes a fixed LR budget across rank-one components using training-time signals, and instantiate it as AnLR-LoRA with signals of velocity and Adam SNR ratio. AnLR-LoRA enables low-velocity rank-one components with reliable update directions to actively engage in the training process.
    \item Extensive experiments show that adaptive anisotropic learning rate encourages utilization of rank budget and improves the adaptation of LoRA across various settings. This benefit remains robust under varying learning rate and transfers to LoRA variants.
\end{itemize}

\section{Related Work}
\label{sec:related}
\subsection{Parameter-Efficient Fine-Tuning}

Parameter-efficient fine-tuning (PEFT) reduces fine-tuning cost by training only a small subset of model parameters. Representative approaches include adapter modules inserted between transformer layers \citep{houlsby2019parameter, chen2022adaptformer}, prompt and prefix tuning that prepend learnable tokens to the input or to the attention state \citep{li2021prefix, jia2022visual}, and methods that perform full-parameter training inside a low-rank gradient subspace \citep{zhao2024galore, chen2026fira}. PEFT modules can also be composed into expandable mixtures of experts for continual learning \citep{wang2025self, zhao2026token}. We focus on low-rank adaptation, the standard PEFT approach for LLMs, reviewing its variants and optimization choices below.

\subsection{Architectural Variants of LoRA}

LoRA \citep{hu2021lora} introduces a low-rank update to a frozen pretrained model, parameterized by two factor matrices initialized so that the update is zero at the start of training. DoRA modify this parameterization by separating magnitude and direction \citep{dora}. Initialization-based variants place the factors on informative starting directions, from the principal or minor singular components \citep{pissa, milora} to the first-step gradient \citep{wang2024loraga, zhang2026primacy, lora-dash}. Some methods adjust the effective rank during training, such as SoRA \citep{sora} with learnable gates, AdaLoRA \citep{adalora} with singular-triplet pruning, and DyLoRA \citep{dylora} with sub-rank sampling. Low-rank structure is similarly exploited in continual learning, where rank-one components serve as incremental experts \citep{lu2026little} and rank minimization acts as an implicit regularizer against forgetting \citep{lu2024take}. Across these designs, all retain LoRA's single global learning rate, applied uniformly across all parameters.

\begin{figure*}[htp]
    \centering
    \begin{subfigure}{0.33\linewidth}
        \centering
        \includegraphics[width=\linewidth]{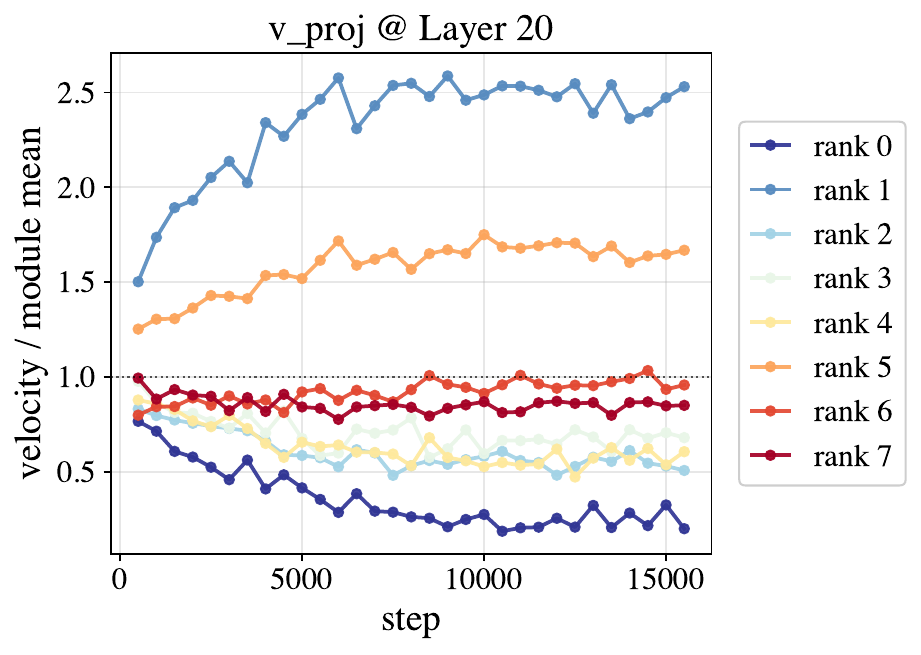}
        \caption{Visualization of per-rank-one velocity across training steps. 
        }
        \label{fig:fig_consistency}
    \end{subfigure}
    \hfill
    \begin{subfigure}{0.31\linewidth}
        \centering
        \includegraphics[width=\linewidth]{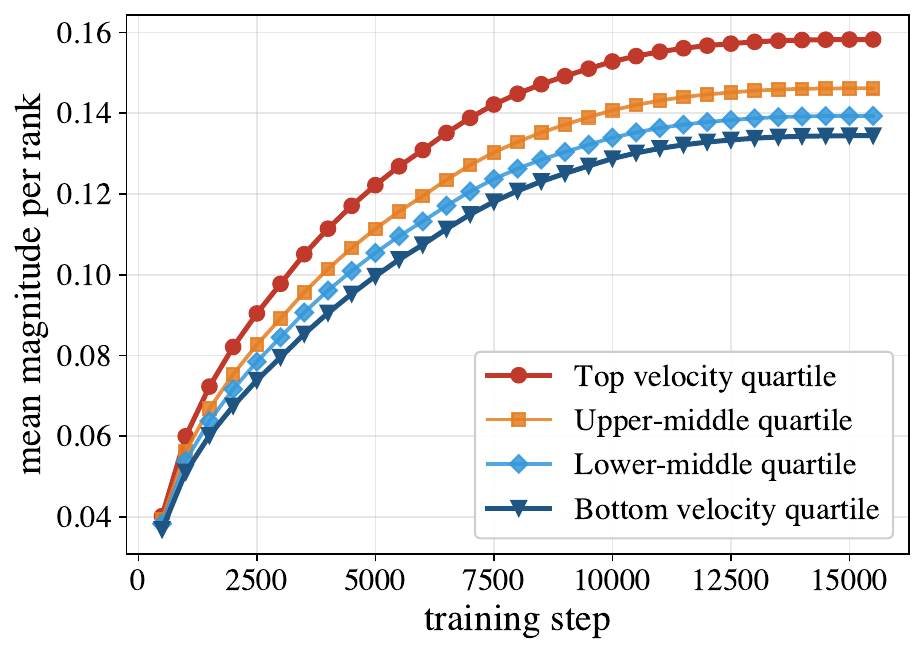}
        \caption{Rank-one velocity vs magnitude accumulated over training.}
        \label{fig:fig_rich_get_richer}
    \end{subfigure}
    \hfill
    \begin{subfigure}{0.31\linewidth}
        \centering
        \includegraphics[width=\linewidth]{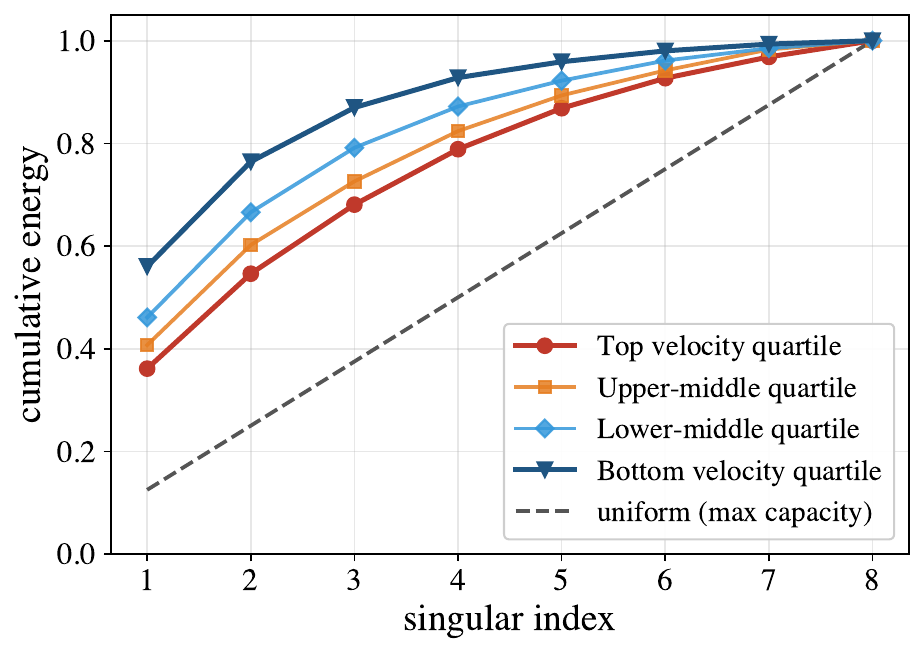}
        \caption{Module velocity vs singular spectrum concentration at convergence.}
        \label{fig:fig_motivation_quartile_spectrum}
    \end{subfigure}
    \caption{Per-rank-one anisotropy in LoRA training on math reasoning. (a) The within-module velocity ordering is established early and persists in training. (b) High-velocity components accumulate larger contribution magnitude throughout training. (c) Lower mean per-rank-one velocity is associated with more concentrated singular spectra of $BA$ at convergence, under-utilizing the available rank budget.}
    \label{fig:fig_velocity_analysis}
\end{figure*}

\subsection{Optimization and Hyperparameter Choices for LoRA}

A parallel line studies how LoRA is trained rather than parameterized. The learning rate is the load-bearing hyperparameter across LoRA variants once well-tuned \citep{lee2026learning}. \citet{loraplus} observe that the two factor matrices play asymmetric roles during fine-tuning and assign each its own learning rate. \citet{chen2026learning} characterize how the global rate should scale with adapter rank, and rsLoRA \citep{rslora} introduces a rank-stabilized output scaling. Other methods modify the optimizer update itself: LoRA-Pro \citep{lorapro}, AltLoRA \citep{yu2026altlora} and RefLoRA \citep{zhang2026reflora} adjust the per-step gradient or factorization to better approximate a full fine-tuning step. These works operate at module or matrix granularity and none assigns a separate learning rate per rank component or derives one from training-time signals. LoRA+ \citep{loraplus} is closest to ours in showing that per-component learning rates outperform the single-LR convention, but its two matrix-level learning rates are fixed at training start. As we show in Sec.~\ref{sec:motivation}, rank-one components within a matrix remain heterogeneous throughout training, so per-matrix allocation discards information preserved by per-rank learning rates. We derive these rates from two training-time signals: the rate of change of each rank's contribution and the consistency of its Adam gradient.

\section{Analysis of Rank-One Training Dynamics in LoRA}
\label{sec:motivation}
This analysis provides the empirical basis for the intervention proposed in Sec.~\ref{sec:method}. We establish three observations from a LoRA fine-tuning run of LLaMA2-7B on MetaMathQA. The analysis covers $224$ LoRA modules in total. We refer to the Frobenius-norm rate of change of each rank-one contribution $B_{:,k}A_{k,:}$ as its \emph{per-rank-one velocity}, defined formally in Sec.~\ref{sec:method:velocity}. We further define its \emph{per-rank-one magnitude} as the cumulative Frobenius norm of $B_{:,k}A_{k,:}$ at each training step. These measures allow us to characterize how adaptation is distributed across rank-one components, both dynamically during optimization and cumulatively over training.

\subsection{Per-Rank-One Velocity is Anisotropic within Each Module}
\label{sec:motivation:anisotropic}

Within a single LoRA adapter, the $r$ rank-one components do not update at the same rate. Fig.~\ref{fig:fig_bars} shows the normalized update velocity of individual rank-one components in LoRA modules, with the dashed line indicating the module-wise mean. The update activity is clearly uneven. Some ranks move much faster than the module average, whereas others remain substantially slower. Across 224 LoRA modules in our analysis, the median max/min velocity ratio is $2\times$, the 90th percentile reaches $4\times$, and the most imbalanced 1\% exceed $8\times$. This heterogeneity appears across different projection layers, suggesting that LoRA rank-one components are not optimized uniformly under a shared learning rate. 
This anisotropy is \emph{within-module}: every rank inside a LoRA adapter shares the same global learning rate by construction, 
so the imbalance cannot be addressed by conventional uniform learning-rate scheduling. A more direct mechanism is therefore to redistribute the learning rate among the $r$ rank-one components within each module, rather than increasing the overall learning rate.

\subsection{The Anisotropy is Established Early and Persists}
\label{sec:motivation:locked-in}

The velocity heterogeneity is not a transient effect. Within each module, the rank-one components that are the most and the least active early in training remain so through to convergence. Fig.~\ref{fig:fig_consistency} shows per-rank-one velocity trajectories in a LoRA module, with each rank colored by its index. The ordering is established early in training, within approximately the first 1k steps and preserved thereafter. The persistent velocity ordering translates directly to a contribution-magnitude ordering. Fig.~\ref{fig:fig_rich_get_richer} shows the mean contribution magnitude of the top-velocity quartile pulls ahead in early training and the gap persists throughout the training process. Under uniform LR there is no mechanism to redistribute capacity, so the imbalance compounds. The ordering is \emph{not predictable from $A^{(0)}, B^{(0)}$ alone}, but is stabilized within the early training stage. The right per-rank-one allocation should therefore be discovered during training and applied throughout, not preset analytically.

\subsection{Low-Velocity Modules Under-Utilize their Nominal Rank Budget}
\label{sec:motivation:concentration}

The within-module magnitude imbalance has a further consequence: modules whose rank-one components evolve most slowly under uniform LR end up with the most concentrated singular spectra in the learned composite update $BA$, packing their squared-Frobenius energy into only a few singular directions and leaving the rest of the nominal rank budget effectively unused. As shown in Fig.~\ref{fig:fig_motivation_quartile_spectrum}, low-velocity modules exhibit a much sharper cumulative energy curve than high-velocity modules, indicating that their learned update is dominated by a small number of directions. In other words, 
despite having the same nominal rank budget, slow-moving modules exploit only a fraction of the available capacity under uniform learning rate. 

The concentration is induced by the optimization process rather than being an inherent property of the downstream update. Full fine-tuning provides a reference for what an update at this rank can look like. In a full fine-tuning experiment on Commonsense170K with LLaMA2-7B, the best rank-32 approximation to the full fine-tuning update has an effective rank of 27.2 out of 32, whereas LoRA at the same rank learns an update with an effective rank of only 12.7. Concentration under uniform LR thus leaves rank capacity under-used and limits how much of the update energy a LoRA adapter can capture, as discussed in Appendix~\ref{appendix:spectrum}. The spectral concentration is the \emph{handle} our intervention pulls on: by redistributing learning rate toward slower rank-one components, AnLR-LoRA aims to encourage broader use of the available rank capacity. We verify in Sec.~\ref{sec:exp} whether this translates into downstream task gains.

\section{Methodology}
\label{sec:method}
Sec.~\ref{sec:motivation} shows that LoRA rank-one components exhibit anisotropic training dynamics within each module, and that this anisotropy emerges early and persists throughout fine-tuning. We now formalize a learning-rate scheme that responds to this anisotropy. The proposed \emph{adaptive anisotropic LR model} attaches a per-rank-one multiplier $s_k^{(t)}$ to the global learning rate within each adapter, redistributing the module-level LR across the $r$ rank-one components. AnLR-LoRA instantiates this model by computing $s_k^{(t)}$ online from the velocity and Adam SNR signals available during Adam optimization.

\subsection{The Uniform-LR Convention}
\label{sec:method:baseline}

Standard LoRA fine-tuning trains $\theta = \{A, B\}$ for each adapter under AdamW with a single global learning rate schedule $\eta_t$ and decoupled weight-decay coefficient $\lambda$. At step $t$, AdamW maintains exponential moving averages of the gradient and its element-wise square,
\begin{equation*}
    m^{(t)} = \beta_1 m^{(t-1)} + (1 - \beta_1) g^{(t)}, 
\end{equation*}
\begin{equation*}
    v^{(t)} = \beta_2 v^{(t-1)} + (1 - \beta_2) g^{(t)} \odot g^{(t)},
\end{equation*}
with bias-corrected estimates $\hat m^{(t)} = m^{(t)} / (1 - \beta_1^t)$ and $\hat v^{(t)} = v^{(t)} / (1 - \beta_2^t)$. The AdamW update direction is
\begin{equation}
    u^{(t)} \;=\; \frac{\hat m^{(t)}}{\sqrt{\hat v^{(t)}} + \varepsilon} \;+\; \lambda\, \theta^{(t)},
    \label{eq:adamw-direction}
\end{equation}
the sum of a preconditioned momentum term and a decoupled weight-decay offset. The parameter update is
\begin{equation}
    \theta^{(t+1)} \;=\; \theta^{(t)} - \eta_t\, u^{(t)}.
    \label{eq:adamw-update}
\end{equation}

In this formulation, the \emph{per-parameter effective learning rate} is $\eta_{t, \theta_i} = \eta_t$ for every parameter $\theta_i$ in every rank-one component of every adapter, because the only place $\eta_t$ enters is as the outer scalar multiplying $u^{(t)}$. We call this the \emph{uniform-LR convention}. It assumes both \emph{uniform granularity}, where all $r$ rank-one components of an adapter share a single LR, and \emph{uniform schedule shape}, where the same global LR schedule $\eta_t$ is applied to every parameter.

\subsection{Adaptive Anisotropic Learning-rate}
\label{sec:method:model}

The analysis in Sec.~\ref{sec:motivation} motivates two changes to the uniform-LR convention: adaptivity and granularity. \emph{Adaptivity} means that the learning rate should depend on the training-time state, rather than being fixed by a global schedule alone. This follows the broader principle of preconditioned optimization~\citep{kingma2014adam,duchi2011adaptive,you2017large,you2019large}, but has not been explored at the level of individual LoRA rank-one components. \emph{Granularity} refers to the unit of learning-rate control, which should match the structure of the LoRA update. Since the LoRA update decomposes into $\Delta W = BA = \sum_{k=1}^{r} B_{:,k} A_{k,:}$, the rank-one component is the finest structured unit of the LoRA parameterization and is where the empirical anisotropy emerges. Per-parameter~\citep{kingma2014adam,huang2024allora} and per-matrix~\citep{loraplus} granularities leave it unaddressed.

We attach a positive, time-varying multiplier $s_k^{(t)}$ to each rank-one component $\theta_k = (A_{k,:},\, B_{:,k})$. Let $u_k^{(t)}$ denote the restriction of the AdamW direction in Eq.~\ref{eq:adamw-direction} to the parameters of component $k$. The adaptive-LR model uses a per-component \emph{effective learning rate} $\eta_{t,k} = s_k^{(t)} \eta_t$ with the update
\begin{equation}
    \theta_k^{(t+1)} \;=\; \theta_k^{(t)} - s_k^{(t)}\, \eta_t\, u_k^{(t)},
    \quad k = 1, \dots, r,
    \label{eq:adaptive-update}
\end{equation}
which multiplies $\eta_t$ by $s_k^{(t)}$  in Eq.~\ref{eq:adamw-update} and leaves AdamW's state unchanged. 
Guided by these observations, we impose two requirements on the multipliers. First, because the anisotropy is within-module, the multipliers should \emph{redistribute} the module-level learning rate rather than change its average scale. We therefore enforce $ \mathbb{E}_k\!\big[ s_k^{(t)} \big] \;=\; 1$ within each module. Second, because the rank ordering emerges during training, the multipliers should be computed online from training-time state, $s_k^{(t)} = \phi(\mathcal{S}_k^{(t)})$, without introducing extra trainable parameters or additional forward/backward passes. Vanilla LoRA is recovered as the special case $s_k^{(t)} = 1$, while coarser schemes such as LoRA+ can be viewed as fixed, matrix-level alternatives. AnLR-LoRA instantiates this model by specifying the state $\mathcal{S}_k^{(t)}$ and the mapping function $\phi$ in the following section.

\subsection{\method{}: Instantiating the Model with Two Training-Time Signals}
\label{sec:method:instantiation}

AnLR-LoRA instantiates the adaptive anisotropic LR model by defining the state $\mathcal{S}_k^{(t)}$ with two complementary training-time signals and specifying the mapping $\phi$ as a log-additive and mean-normalized function.

\paragraph{What information should the multiplier use?}
To decide whether to boost or dampen rank-one component $k$, $s_k^{(t)}$ should capture two aspects of its training dynamics: how actively the component is currently changing, and whether its change direction is reliable. These two aspects are complementary. A slow-moving component may follow a consistent but under-exploited direction, in which case increasing its LR is beneficial. Alternatively, it may receive noisy gradients with no stable direction, in which case boosting it too much would amplify noise. We therefore use velocity to measure component activity and Adam SNR to measure update reliability.

\paragraph{Velocity: function-space activity.}
\label{sec:method:velocity}
The $k$-th rank-one component contributes the rank-one matrix $B_{:,k} A_{k,:}$ to the LoRA update. We define its \emph{velocity} as the Frobenius norm of the instantaneous change in this contribution,
\begin{equation}
  \nu_k \;=\; \left\| \frac{d}{dt}\big( B_{:,k} A_{k,:} \big) \right\|_F,
  \label{eq:velocity-cont}
\end{equation}
approximated from the current LoRA gradients $\nabla_A L, \nabla_B L$ via the product rule:
\begin{equation}
  \nu_k^{(t)} \;=\; \left\| (\nabla_B L)_{:,k}^{(t)}\, A_{k,:}^{(t)} \;+\; B_{:,k}^{(t)}\, (\nabla_A L)_{k,:}^{(t)}
\right\|_F.
  \label{eq:velocity-discrete} 
\end{equation}
  
To reduce step-level noise, we maintain an EMA-smoothed velocity $\bar{\nu}_k^{(t)} = \rho \bar{\nu}_k^{(t-1)} + (1-\rho)\nu_k^{(t)}$.
Velocity is defined on the \emph{output} of the rank-one component rather than on its parameters, so it captures the multiplicative coupling between $A_{k,:}$ and $B_{:,k}$, which is not reflected by per-parameter statistics alone. It is computable from quantities the optimizer step already produces, without additional forward/backward passes.

\paragraph{Adam SNR: parameter-space confidence.}
\label{sec:method:snr}
Let $\hat{m}$ and $\hat{v}$ denote Adam's bias-corrected first and second moment estimates. 
For each parameter, the element-wise SNR is defined as $\mathrm{snr}(\theta_i) =\frac{|\hat{m}_i|}{\sqrt{\hat{v}_i} + \varepsilon}$,
where $\varepsilon$ is a small constant for numerical stability. For rank component $k$, we aggregate the SNR values over the corresponding row of $A$ and column of $B$ and obtain
\begin{equation}
    \mathrm{snr}_k
    =
    \left(
    \left\|
    \frac{|\hat{m}_{A_{k,:}}|}
    {\sqrt{\hat{v}_{A_{k,:}}}+\varepsilon}
    \right\|_2^2
    +
    \left\|
    \frac{|\hat{m}_{B_{:,k}}|}
    {\sqrt{\hat{v}_{B_{:,k}}}+\varepsilon}
    \right\|_2^2
    \right)^{1/2}.
\label{eq:snr}
\end{equation}

High $\mathrm{snr}_k$ indicates that the gradient direction has been consistent across recent batches, identifying a stable update direction, while low $\mathrm{snr}_k$ indicates noisy and less reliable updates. The SNR can be obtained directly from AdamW's optimizer state and \method{} merely \emph{promotes} this quantity from an implicit step modulator to an explicit per-component meta-signal.

\begin{table*}[htp]
\centering
\resizebox{\linewidth}{!}{%
\begin{tabular}{llc cccccccc c}
\toprule
\textbf{Model} & \textbf{Method} & \textbf{\#Params (\%)} & \textbf{BoolQ} & \textbf{PIQA} & \textbf{SIQA} & \textbf{HellaSwag} & \textbf{WinoGrande} & \textbf{ARC-e} & \textbf{ARC-c} & \textbf{OBQA} & \textbf{Avg.} \\
\midrule
\multirow{4}{*}{\shortstack[l]{LLaMA2-7B}}
 & LoRA            & 0.83 & 72.2 & \textbf{84.6} & 77.4 & 90.1 & 82.6 & 81.2 & 67.4 & 64.8 & 77.5 \\
 & rsLoRA          & 0.83 & 71.1 & 81.6 & 78.8 & 87.7 & 81.6 & 82.6 & 67.2 & 81.4 & 79.0 \\
 & LoRA+           & 0.83 & 70.7 & 81.8 & 78.8 & 89.5 & 82.1 & 83.8 & 68.6 & 79.8 & 79.4  \\
 & AnLR-LoRA      & 0.83 & \textbf{72.5} & 84.1 & \textbf{78.9} & \textbf{91.4} & \textbf{82.8} & \textbf{85.0} & \textbf{71.3} & \textbf{82.4} & \textbf{81.0} \\
\midrule
\multirow{4}{*}{\shortstack[l]{LLaMA3-8B}}
 & LoRA  & 0.70 & 70.8 & 85.2 & 79.9 & 91.7 & 84.3 & 84.2 & 71.2 & 79.0 & 80.8 \\
 & rsLoRA          & 0.70 & 72.0 & 85.8 & 79.7 & 92.8 & 83.7 & 85.9 & 73.7 & 81.2 & 81.9 \\
 & LoRA+           & 0.70 & 73.8 & \textbf{88.5} & 80.0 & 95.1 & 86.8 & \textbf{89.9} & 78.6 & \textbf{85.6} & 84.8 \\
 & AnLR-LoRA      & 0.70 & \textbf{75.2} & 88.2 & \textbf{80.5} & \textbf{95.3} & \textbf{87.2} & \textbf{89.9} & \textbf{79.6} & 85.0 & \textbf{85.1} \\
\midrule
\multirow{4}{*}{\shortstack[l]{Qwen2.5-7B}}
 & LoRA            & 0.71 & 64.3 & \textbf{89.8} & 80.5 & 93.7 & 84.5 & 95.0 & 86.7 & \textbf{91.4} & 85.7 \\
 & rsLoRA          & 0.71 & 73.8 & 88.4 & 79.3 & 93.8 & 83.2 & 93.1 & 84.1 & 89.4 & 85.6 \\
 & LoRA+           & 0.71 & 66.2 & 88.7 & 79.4 & 95.0 & 85.2 & 94.4 & 86.4 & 90.2 & 85.7 \\
 & AnLR-LoRA      & 0.71 & \textbf{74.3} & 88.7 & \textbf{80.6} & \textbf{95.3} & \textbf{86.9} & \textbf{95.1} & \textbf{87.5} & 90.8 & \textbf{87.4} \\
\bottomrule
\end{tabular}%
}
\caption{Commonsense reasoning accuracy (\%) with LLaMA2-7B, LLaMA3-8B and Qwen2.5-7B. 
}
\label{tab:commonsense}
\end{table*}

\paragraph{Log-additive fusion.}
The two signals are combined multiplicatively in the multiplier space. Velocity suppresses the multiplier for fast-moving components, while SNR increases it for components with reliable update directions. This suggests an additive formulation in log space. After per-module mean-normalization on $\bar{\nu}_k$ and $\mathrm{snr}_k$, we obtain $\tilde\nu_k$ and $\widetilde{\mathrm{snr}}_k$, and the \method{} multiplier is defined as
\begin{equation}
    \log s_k^{(t)} \;=\; \alpha\, \log \tfrac{1}{\tilde \nu_k^{(t)}} \;+\; \beta\, \log \widetilde{\mathrm{snr}}_k^{(t)}.
    \label{eq:fusion}
\end{equation}
The velocity term boosts slow-moving components and dampens fast-moving ones, while the SNR term favors components with more reliable update directions. For numerical stability, we clamp the log-multiplier to $[-\log\kappa, \log\kappa]$ with a clamp bound $\kappa>1$. The resulting multipliers are then normalized within each module to satisfy $\mathbb{E}_k[\eta_t \cdot s_k] = \eta_t$. In practice, we activate anisotropic scheduling after a short warmup period, during which all multipliers are set to $s_k^{(t)}=1$. This avoids using unstable early-step velocity and SNR estimates before the optimizer statistics become reliable. Thus, AnLR-LoRA redistributes the learning rate across rank components rather than amplifying the overall learning rate.

\subsection{Application of Per-Rank-One LR via Post-Step Delta Scaling}
\label{sec:method:implementation}

We apply the effective per-component learning rate $s_k^{(t)}\eta_t$ through post-step delta scaling. At each step, we first perform a standard AdamW update 
\begin{equation}
    \hat{\theta}_k^{(t+1)}
    =
    \theta_k^{(t)} - \eta_t u_k^{(t)}.
\end{equation}
We then rescale the resulting per-component parameter delta:
\begin{equation}
    \theta_k^{(t+1)}
    =
    \theta_k^{(t)}
    +
    s_k^{(t)}
    \bigl(
    \hat{\theta}_k^{(t+1)} - \theta_k^{(t)}
    \bigr)
    =
    \theta_k^{(t)}
    -
    s_k^{(t)} \eta_t u_k^{(t)} .
    \label{eq:post-step-update}
\end{equation}
This update directly realizes the adaptive-LR model in Eq.~\ref{eq:adaptive-update}. For fixed AdamW direction $u_k^{(t)}$, it is equivalent to using the effective per-component learning rate $s_k^{(t)}\eta_t$.

We use post-step scaling for two reasons. First, $s_k^{(t)}$ is computed from signals of the current optimizer step and cannot be preset in advance. Second, scaling the parameter delta avoids the cancellation that would occur if the gradient were scaled before AdamW normalization. The implementation is lightweight. It requires one parameter snapshot and one element-wise rescaling per step, while leaving AdamW's first- and second-moment states unchanged.

\begin{table}[t]
\centering
\resizebox{0.95\linewidth}{!}{%
\begin{tabular}{lccc}
\toprule
\textbf{Method}  & \textbf{MT-Bench} & \textbf{GSM8K} & \textbf{HumanEval} \\
\midrule
LoRA               & 6.11          & 59.59          & 24.39          \\
    LoRA+             & 5.98             & 60.27          & 24.39        \\
    AnLR-LoRA          & \textbf{6.36}           & \textbf{61.87}  & \textbf{28.05}    \\
\bottomrule
\end{tabular}
}
  \caption{Evaluation results for dialogue, math and coding with LLaMA2-7B.}
  \label{tab:nlg_results}
\end{table}

\begin{table*}[t]
  \centering
  \small
  \setlength{\tabcolsep}{5pt}
  \begin{tabular}{lccccccccc}
    \toprule
    \textbf{Method} & \textbf{\#Params (\%)}
      & \textbf{VQA$^{\text{v2}}$} & \textbf{GQA} & \textbf{VizWiz} & \textbf{SQA} & \textbf{VQA$^{\text{T}}$} & \textbf{POPE} & \textbf{MMBench} & \textbf{Avg.} \\
    \midrule
    FT$^\dagger$ & 100 & 78.50 & 61.90 & 50.00 & 66.80 & 58.20 & 85.90 & 64.30 & 66.50 \\
    \midrule
    LoRA & 4.61 & \textbf{79.02}  & \textbf{63.34} & 47.81 & 69.06 & 57.56 & 86.09 & \textbf{74.96} & 68.26 \\
    LoRA+ & 4.61 & 78.59 & 62.86 & 47.18 & 66.39 & 56.56 & 86.88 & 73.36 & 67.40 \\
    AnLR-LoRA & 4.61 & \textbf{79.02} & 63.07 & \textbf{50.88} & \textbf{69.11}  & \textbf{58.21}  & \textbf{86.89} & 74.43 & \textbf{68.80} \\
    \bottomrule
  \end{tabular}
  \caption{%
    Results on visual
    instruction tuning with LLaVA-1.5-7B. FT results cited from \citet{dora}.
    }
  \label{tab:visual_instruction_tuning}
\end{table*}

\begin{figure*}
    \centering
    \begin{subfigure}{0.31\linewidth}
        \centering
        \includegraphics[width=\linewidth]{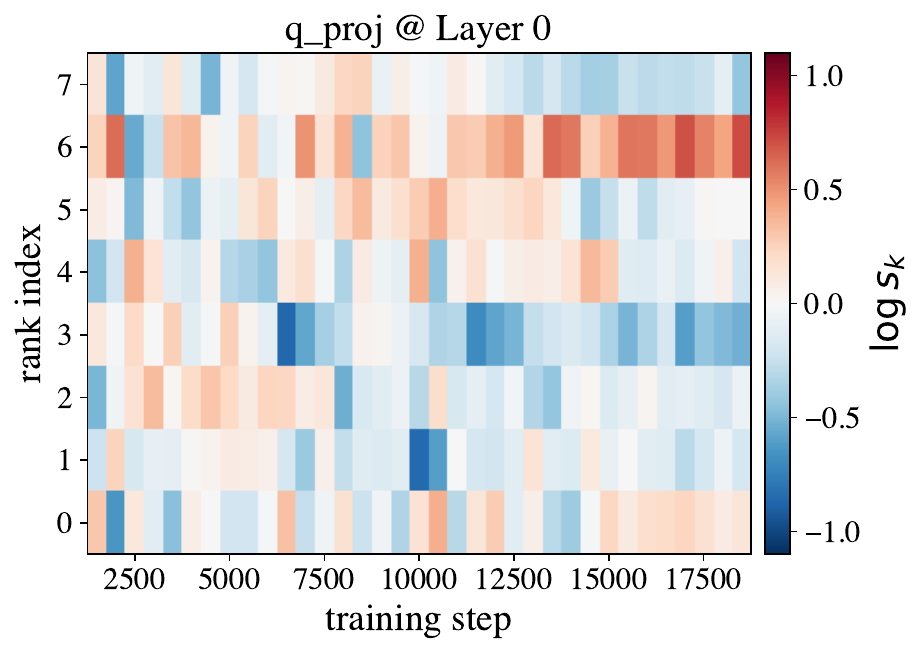}
        \caption{Per-rank-one LR multiplier $\log s_k$ over training.}
        \label{fig:efflr_heatmap}
    \end{subfigure}
    \hfill
    \begin{subfigure}{0.31\linewidth}
        \centering
        \includegraphics[width=\linewidth]{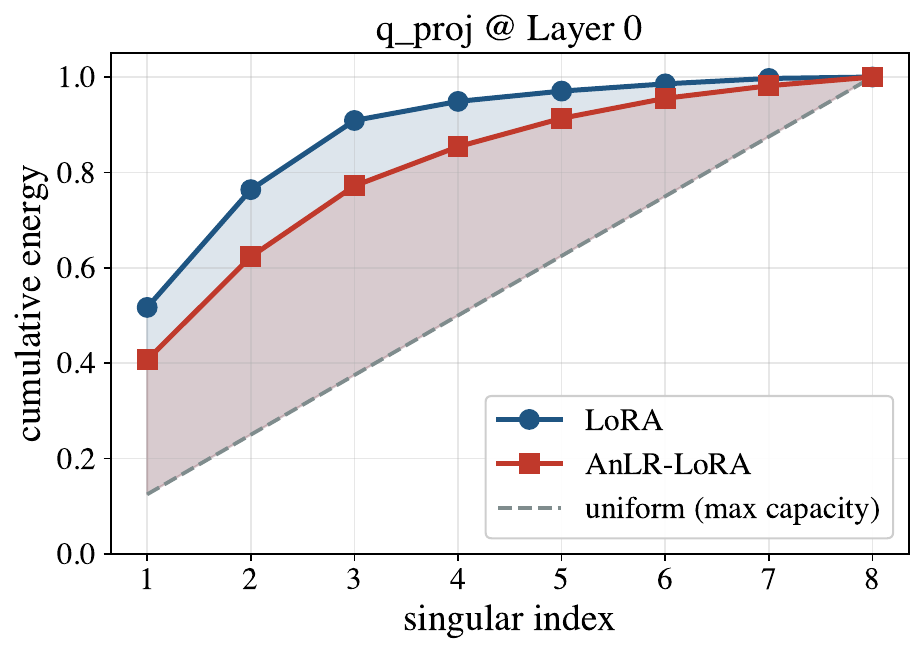}
        \caption{Singular spectrum concentration of the update $BA$ at convergence.}
        \label{fig:flatten_cum}
    \end{subfigure}
    \hfill
    \begin{subfigure}{0.31\linewidth}
        \centering
        \includegraphics[width=\linewidth]{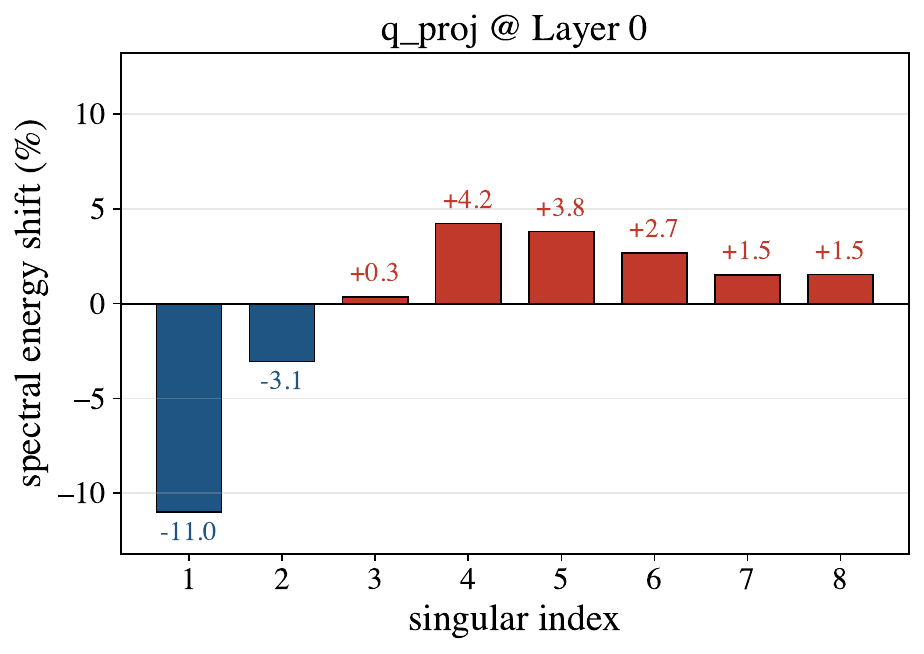}
        \caption{Per-index spectral energy shift from LoRA to AnLR-LoRA.}
        \label{fig:spectral_energy_shift}
    \end{subfigure}
    \caption{Effect of AnLR-LoRA's anisotropic learning rate.}
    \label{fig:lr_sched_effect}
\end{figure*}

\begin{figure}
    \centering
    \includegraphics[width=0.75\linewidth]{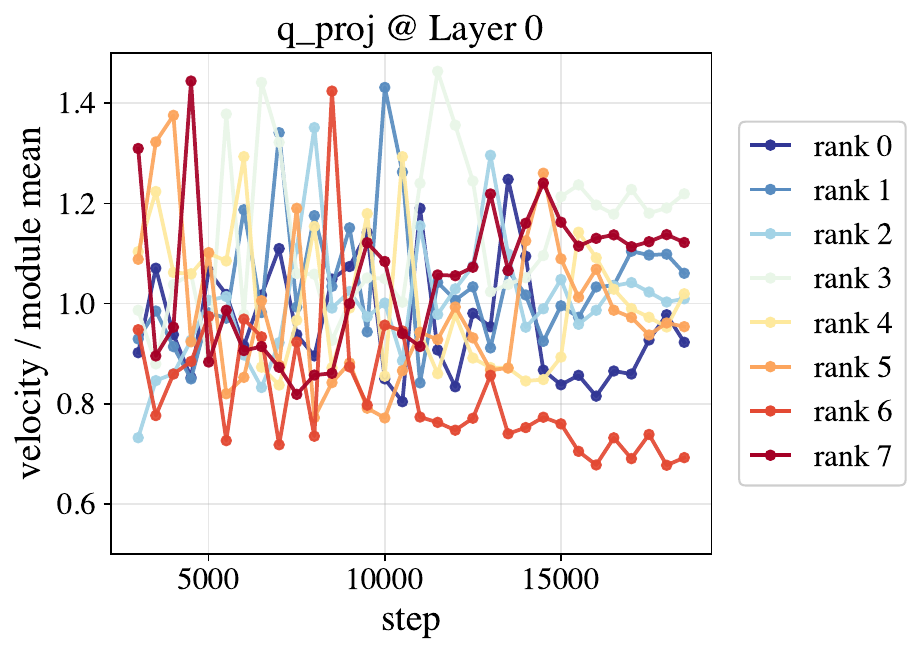}
    \caption{Visualization of per-rank-one velocity of AnLR-LoRA.}
    \label{fig:per_rank_lora_consistency}
\end{figure}

\section{Experiments}
\label{sec:exp}

We evaluate AnLR-LoRA across commonsense reasoning, math reasoning, code generation, dialogue and visual instruction tuning, spanning $18$ datasets across four backbones, including LLaMA2-7B~\citep{touvron2023llama}, LLaMA3-8B~\citep{llama3modelcard}, Qwen2.5-7B~\citep{qwen2025qwen25technicalreport} and LLaVA-1.5-7B~\citep{liu2024improvedbaselinesvisualinstruction}. We compare AnLR-LoRA against LoRA~\citep{hu2021lora} as well as the optimization-related variants rsLoRA~\citep{rslora} and LoRA+~\citep{loraplus}. All experiments run on NVIDIA H100 and H200 GPUs. Full implementation details are in Appendix~\ref{app:impl_details}.

\subsection{Main Results}
\label{sec:exp:main}
\paragraph{Commonsense Reasoning.}
We finetune three language backbones on Commonsense170K~\citep{hu2023llm} and evaluate on eight tasks: BoolQ~\citep{clark2019boolq}, PIQA~\citep{bisk2020piqa}, SIQA~\citep{sap2019socialiqa}, HellaSwag~\citep{zellers2019hellaswag}, WinoGrande~\citep{sakaguchi2021winogrande}, ARC-e, ARC-c~\citep{clark2018think} and OBQA~\citep{mihaylov2018can}. As shown in Tab.~\ref{tab:commonsense}, AnLR-LoRA achieves the best average accuracy on all three backbones, outperforming standard LoRA by up to $4.3\%$ under the same parameter and learning-rate budget. It also consistently surpasses the optimization-related baselines rsLoRA and LoRA+ across models of different families and sizes, suggesting that the improvement comes from redistributing a fixed learning-rate budget across rank-one components rather than tuning a single global learning rate.

\paragraph{Natural Language Generation.}
We further finetune LLaMA2-7B on the 100K subset of WizardLM~\citep{xu2025wizardlmempoweringlargepretrained}, MetaMathQA~\citep{metamath} and Code-Feedback~\citep{codefeedback}, and evaluate AnLR-LoRA on dialogue with MT-Bench~\citep{zheng2023judgingllmasajudgemtbenchchatbot}, math reasoning with GSM8K~\citep{gsm8k} and code generation with HumanEval~\citep{humaneval}, respectively. As shown in Tab.~\ref{tab:nlg_results}, AnLR-LoRA improves over both LoRA and LoRA+ on all three tasks, with the largest gain on HumanEval ($+3.66\%$ over LoRA). The gains suggest that adaptive anisotropic learning-rate also benefits open-ended generation.

\paragraph{Visual Instruction Tuning.}
We extend AnLR-LoRA to the multimodal setting by finetuning LLaVA-1.5-7B on the instruction tuning dataset~\citep{liu2024improvedbaselinesvisualinstruction} and evaluating on seven vision-language benchmarks. As shown in Tab.~\ref{tab:visual_instruction_tuning}, AnLR-LoRA achieves the best average score, matching or surpassing LoRA on most benchmarks and obtaining the largest gain on VizWiz ($+3.07\%$ over LoRA). The improvement remains consistent at a much larger rank ($r=128$) and in the multimodal setting, showing that adaptive LR remains effective beyond standard language tasks.

\subsection{Ablation Studies and Analyses}

\noindent\textbf{AnLR-LoRA encourages broader use of rank budget. } We analyze how AnLR-LoRA changes the training dynamics of LoRA fine-tuning. Fig.~\ref{fig:efflr_heatmap} shows that AnLR-LoRA produces per-rank-one multipliers that vary across training steps, rather than assigning a fixed importance ordering. A rank-one component can be boosted or dampened at different stages of training, indicating that the method adapts to the current optimization state instead of imposing a static rank prior. This dynamic redistribution encourages different rank-one components to participate throughout training, rather than allowing a small subset of rank-one components to consistently dominate the update (Fig.~\ref{fig:per_rank_lora_consistency}).

\begin{figure}
    \centering
    \includegraphics[width=0.75\linewidth]{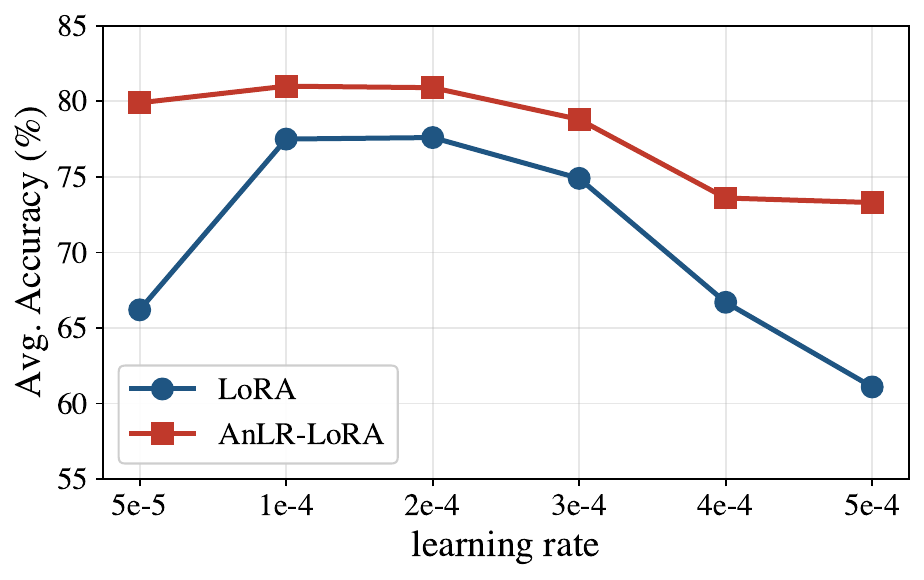}
    \caption{Average commonsense reasoning accuracy on LLaMA2-7B across varying LR. AnLR-LoRA exceeds LoRA at every LR, with the gap widening as LR moves away from the optimal range.}
    \label{fig:fig_lr_sweep}
\end{figure}

This anisotropic LR schedule also changes the structure of the learned update. Compared to LoRA, which concentrates update energy in the leading singular directions, AnLR-LoRA produces a flatter cumulative energy curve (Fig.~\ref{fig:flatten_cum}) and redistributes spectral energy toward lower singular directions (Fig.~\ref{fig:spectral_energy_shift}). This indicates that AnLR-LoRA reduces spectral concentration in the learned update and encourages broader use of the available rank capacity.

\noindent\textbf{Analysis on varying LR. } 
Fig.~\ref{fig:fig_lr_sweep}
examines whether the gains of AnLR-LoRA can be explained by simply tuning the global learning rate. We compare LoRA and AnLR-LoRA on LLaMA2-7B across a wide range of learning rates. LoRA is highly sensitive to the choice of LR and its average accuracy varies by 16.5 points and peaks at 77.6. Increasing or decreasing the global learning rate can improve some tasks but degrade others, 
and performance deteriorates when the learning rate moves away from a narrow favorable range (Tab.~\ref{tab:lr_sweep}). 
This indicates that simply changing the overall learning rate is not sufficient to resolve the uneven rank-one learning dynamics observed in Sec.~\ref{sec:motivation}. 

AnLR-LoRA consistently outperforms LoRA across varying learning rates. The gain over LoRA widens as the learning rate moves away from the optimal range, which shows AnLR-LoRA degrades much more gracefully under poorly tuned LR. The robustness across LR settings highlights the advantage of adaptive anisotropic LR as it provides a finer-grained optimization mechanism with gains that no single global LR can achieve.

\begin{table}[htp]
  \centering
  \resizebox{0.9\linewidth}{!}{%
  \begin{tabular}{lccc}
    \toprule
    \textbf{Method} &  \textbf{GSM8K} & \textbf{HumanEval} \\
    \midrule
    PiSSA                          & 58.98          & 24.39          \\
    AnLR-PiSSA                      & \textbf{59.36} & \textbf{25.00} \\
    \midrule
    MiLoRA                         & 60.12          & 25.61          \\
    AnLR-MiLoRA                      & \textbf{61.41} & \textbf{26.22} \\
    \midrule
    LoRA-Dash                     & 60.73 & 27.44          \\
    AnLR-LoRA-Dash                   & \textbf{61.11 }        & \textbf{28.05} \\
    \bottomrule
  \end{tabular}
  }
  \caption{Results of applying anisotropic learning rate to LoRA variants.}
  \label{tab:lora_variants}
\end{table}

\begin{table}[htp]
  \centering
  \resizebox{0.95\linewidth}{!}{%
  \begin{tabular}{lcc}
    \toprule
    \textbf{LR scaling}   & \textbf{GSM8K} & \textbf{HumanEval} \\
    \midrule
    None (uniform)                 & 59.59          & 24.39          \\
    Velocity-only    &   61.18  &  26.22 \\
    SNR-only   & 60.12  & 25.00 \\
    Velocity + SNR (ours)          & \textbf{61.87}  & \textbf{28.05}    \\
    \bottomrule
  \end{tabular}
  }
  \caption{Ablation of AnLR-LoRA's per-rank-one scaling signal.}
  \label{tab:sched_ablation}
\end{table}

\noindent\textbf{Application to other LoRA variants. } 
We extend our evaluation of anisotropic LR to other LoRA variants in Tab.~\ref{tab:lora_variants}. We apply the same schedule to PiSSA~\cite{pissa}, MiLoRA~\cite{milora} and LoRA-Dash~\cite{lora-dash}, covering variants that modify LoRA initialization or update behavior. Anisotropic LR consistently improves all three methods on math reasoning and code generation tasks, suggesting that our approach is largely orthogonal to existing LoRA-based methods and can provide complementary gains.

\paragraph{Ablation on training-time state signals.}
Tab.~\ref{tab:sched_ablation} ablates the signals used for adaptive LR scaling. Velocity-only scaling consistently improves over the uniform baseline, confirming that rank-one update activity is informative for learning-rate redistribution. SNR scaling allocates more learning rate to components with reliable update directions. Combining the two yields the best results on both GSM8K and HumanEval, suggesting that the two signals offer complementary guidance. Appendix~\ref{appendix:mechanism} further verifies this mechanism, showing that the velocity term reduces the within-module velocity imbalance and the SNR term prevents components with noisy update directions from being over-amplified. The full AnLR-LoRA schedule therefore benefits from combining activity-aware redistribution with confidence-aware refinement.

\section{Conclusion}
This paper revisits the uniform-LR convention in LoRA fine-tuning and shows that it overlooks substantial heterogeneity in per-rank-one training dynamics within a single LoRA adapter. We propose an adaptive anisotropic LR model that assigns each rank-one component its own LR based on training-time signals, and instantiate it as AnLR-LoRA using function-space velocity and Adam SNR to encourage learning on low-velocity components with reliable update directions. Experiments across language and multimodal settings demonstrate that AnLR-LoRA improves LoRA adaptation and encourages broader use of rank capacity, highlighting anisotropic LR redistribution as a simple and effective optimization dimension.

\section*{Limitations}
AnLR-LoRA is one instantiation within a broader adaptive anisotropic LR framework, leaving alternative signal choices, fusion strategies and more sophisticated uncertainty measures for update direction as directions for future work. Our rank-utilization analysis verifies that the full fine-tuning update spreads its energy across the rank budget, and characterizing how broadly this holds across tasks would further delineate when anisotropic LR is most beneficial. The extension to broader scenarios, such as multilingual, long-context and larger-scale training, also remains to be explored.

\section*{Acknowledgements}
This work was partially supported by the ARC DECRA Fellowship (DE230101591) and the ARC Discovery Project Grant (DP260103379) awarded to D. Gong, as well as PhD scholarship support from UNSW and CSIRO Data61.

\bibliography{main}

\clearpage
\appendix

\section{More Details about AnLR-LoRA}

\subsection{Pre- and Post-Step Scaling under AdamW}
\label{appendix:proofs}

This section states and proves Propositions~\ref{prop:pre-step} and~\ref{prop:post-step}, which together delimit the design space for per-rank-one effective learning rates under AdamW. Proposition~\ref{prop:pre-step} shows that pre-step gradient scaling is ineffective for this purpose. Proposition~\ref{prop:post-step} shows that post-step delta scaling is equivalent to direct outer-rate substitution.

\paragraph{Setup.}
AdamW's update direction restricted to the $k$-th rank-one component is
\begin{equation}
u_k^{(t)} = \frac{\hat m_k^{(t)}}{\sqrt{\hat v_k^{(t)}} + \varepsilon} + \lambda \theta_k^{(t)},
\label{eq:adamw-direction_appendix}
\end{equation}
and the post-step delta-scaled update is
\begin{equation}
\theta_k^{(t+1)} = \theta_k^{(t)} + s_k^{(t)} \bigl(\hat\theta_k^{(t+1)} - \theta_k^{(t)}\bigr),
\label{eq:post-step-update_appendix}
\end{equation}
with $\hat\theta_k^{(t+1)} = \theta_k^{(t)} - \eta_t u_k^{(t)}$.

\begin{proposition}[Pre-step gradient scaling]
\label{prop:pre-step}
Let $\{\theta^{(t)}\}$ be the AdamW trajectory under gradient sequence $\{g^{(t)}\}$ with learning rate $\eta_t$, moment decays $\beta_1, \beta_2$, decoupled weight decay $\lambda$ and stability constant $\varepsilon$. Replacing the gradient by $c \cdot g^{(t)}$ at every step, for any fixed $c > 0$, produces a trajectory equal to the unscaled one when $\varepsilon = 0$, and differs by an $O(\varepsilon)$ perturbation otherwise.
\end{proposition}

\begin{proof}
Consider two AdamW trajectories initialized identically at $\theta^{(0)}$, with $m^{(0)} = v^{(0)} = 0$ and shared $(\eta_t, \lambda, \beta_1, \beta_2, \varepsilon)$:
\begin{itemize}
\item Run~A driven by $\{g^{(t)}\}_{t \geq 1}$.
\item Run~B driven by $\{c\, g^{(t)}\}_{t \geq 1}$, fixed $c > 0$.
\end{itemize}

\paragraph{Step 1: moment scaling.}
We show by induction that
\begin{equation}
m_B^{(t)} = c\, m_A^{(t)}, \quad v_B^{(t)} = c^2 v_A^{(t)}.
\label{eq:moment-scaling}
\end{equation}
The base case ($t = 0$) is trivial. For the inductive step,
\begin{align*}
m_B^{(t)}
&= \beta_1 m_B^{(t-1)} + (1 - \beta_1) g_B^{(t)} \\
&= c \bigl[\beta_1 m_A^{(t-1)} + (1 - \beta_1) g^{(t)}\bigr] \\
&= c\, m_A^{(t)},
\end{align*}
and similarly $v_B^{(t)} = c^2 v_A^{(t)}$. Bias correction applies symmetrically, so $\hat m_B^{(t)} = c\, \hat m_A^{(t)}$ and $\hat v_B^{(t)} = c^2 \hat v_A^{(t)}$.

\paragraph{Step 2: update direction.}
Let $\sigma_t := \sqrt{\hat v_A^{(t)}}$ for brevity. The AdamW direction in Run~B is
\begin{align*}
u_B^{(t)}
&= \frac{\hat m_B^{(t)}}{\sqrt{\hat v_B^{(t)}} + \varepsilon} + \lambda \theta_B^{(t)} \\
&= \frac{c\, \hat m_A^{(t)}}{c\, \sigma_t + \varepsilon} + \lambda \theta_B^{(t)} \\
&= \frac{\hat m_A^{(t)}}{\sigma_t + \varepsilon/c} + \lambda \theta_B^{(t)}.
\end{align*}

\noindent\textit{Case $\varepsilon = 0$.}
The factor $c$ cancels, giving $u_B^{(t)} = u_A^{(t)}$ whenever $\theta_B^{(t)} = \theta_A^{(t)}$. The parameter update preserves the equality, so induction yields $\theta_B^{(t)} = \theta_A^{(t)}$ for all $t \geq 0$.

\noindent\textit{Case $\varepsilon > 0$.}
$u_B^{(t)}$ differs from $u_A^{(t)}$ only by the substitution $\varepsilon \mapsto \varepsilon/c$:
\begin{equation*}
u_B^{(t)} - u_A^{(t)}
= \frac{\hat m_A^{(t)}\, \varepsilon\, (1 - 1/c)}{(\sigma_t + \varepsilon/c)(\sigma_t + \varepsilon)},
\end{equation*}
which is $O(\varepsilon)$ and negligible whenever $\sigma_t \gg \varepsilon$, the operating regime of AdamW.

\paragraph{Conclusion.}
Pre-step gradient scaling by $c$ yields a trajectory equal to the unscaled one up to $O(\varepsilon)$. A constant factor $s_k$ applied to $g_k^{(t)}$ is therefore cancelled by Adam's preconditioner. A time-varying $s_k^{(t)}$ is not exactly cancelled, but it acts by perturbing the moment estimates rather than by rescaling the step, so it does not install a per-component effective learning rate either.
\end{proof}

\begin{proposition}[Post-step delta scaling]
\label{prop:post-step}
Let $\theta^{(t+1)}_k$ be defined by Eq.~\ref{eq:post-step-update_appendix}, where $\hat\theta^{(t+1)}_k$ is one AdamW step on $\theta^{(t)}_k$ with learning rate $\eta_t$ and decoupled weight decay $\lambda$. The resulting trajectory $\{\theta^{(t)}_k\}$ is identical to the trajectory obtained by running AdamW on the same gradient sequence with outer learning rate $s_k^{(t)} \eta_t$ and unchanged decoupled weight decay $\lambda$.
\end{proposition}

\begin{proof}
Substituting the AdamW step into Eq.~\ref{eq:post-step-update_appendix}:
\begin{align*}
\theta_k^{(t+1)}
&= \theta_k^{(t)} + s_k^{(t)} \bigl(\hat\theta_k^{(t+1)} - \theta_k^{(t)}\bigr) \\
&= \theta_k^{(t)} + s_k^{(t)} \bigl(-\eta_t u_k^{(t)}\bigr) \\
&= \theta_k^{(t)} - \bigl(s_k^{(t)} \eta_t\bigr)\, u_k^{(t)}.
\end{align*}
This is the AdamW update with outer learning rate $\eta_t$ replaced by $s_k^{(t)} \eta_t$, leaving $\hat m_k^{(t)}$, $\hat v_k^{(t)}$ and $\lambda$ unchanged. Expanding via Eq.~\ref{eq:adamw-direction_appendix}:
\begin{align*}
\theta_k^{(t+1)}
&= \theta_k^{(t)} - \bigl(s_k^{(t)} \eta_t\bigr) \frac{\hat m_k^{(t)}}{\sqrt{\hat v_k^{(t)}} + \varepsilon} \\
&\quad - \bigl(s_k^{(t)} \eta_t \lambda\bigr)\, \theta_k^{(t)},
\end{align*}
which matches a standard AdamW update with outer rate $\eta'_t = s_k^{(t)} \eta_t$ and shrinkage coefficient $\eta'_t \lambda$ on $\theta_k^{(t)}$. The moment recurrences for $m_k$ and $v_k$ depend only on the gradient sequence and $(\beta_1, \beta_2)$, not on the outer rate, so they are unchanged by the substitution. Therefore, post-step delta scaling at every step produces the same trajectory as running AdamW on the same gradient sequence with outer learning rate $s_k^{(t)} \eta_t$ and unchanged decoupled weight decay $\lambda$.
\end{proof}

\paragraph{Remarks.}
Together, the two propositions delimit the admissible implementations of per-rank-one effective learning rates under AdamW: pre-step gradient scaling is ineffective for this purpose (Proposition~\ref{prop:pre-step}), while post-step delta scaling and direct per-component outer-rate substitution are equivalent (Proposition~\ref{prop:post-step}). \method{} uses the post-step delta form because it is a drop-in modification of the AdamW step output that requires no change to the optimizer's internal state or update rule.

\subsection{Training Procedure of AnLR-LoRA}
\label{appendix:pseudocode}

Alg.~\ref{alg:anlr_lora} summarizes the training procedure of AnLR-LoRA. At each step, we compute per-rank-one velocity from
the loss gradients on $A$ and $B$ and per-rank-one Adam SNR from Adam's bias-corrected moment estimates. After warmup, the
two normalized signals are fused into per-rank-one multipliers. A standard AdamW step then produces provisional LoRA
parameters, and the multipliers are applied through post-step delta scaling on the resulting updates. This implementation
leaves the AdamW optimizer state unchanged and introduces no additional forward or backward passes.

\begin{algorithm*}[t]
\caption{AnLR-LoRA}
\label{alg:anlr_lora}
\begin{algorithmic}[1]
\Require LoRA parameters $\{A,B\}$, AdamW optimizer, LR schedule $\eta_t$,
warmup $T_{\mathrm{warm}}$, EMA coefficient $\rho$, fusion weights $\alpha,\beta$,
clamp factor $\kappa$
\State Initialize velocity EMA $\bar{\nu}_k \leftarrow 0$ for each rank-one component
\For{each training step $t$}
    \State Save a snapshot of current LoRA parameters $\theta^{(t)}=\{A^{(t)},B^{(t)}\}$
    \State Take one standard AdamW step to obtain provisional parameters $\hat{\theta}^{(t+1)}$
    \For{each LoRA module}
        \State Compute per-rank velocity $\nu_k^{(t)}$ using Eq.~\ref{eq:velocity-discrete}
        \State Update EMA velocity:
        $\bar{\nu}_k^{(t)} \leftarrow \rho\bar{\nu}_k^{(t-1)} + (1-\rho)\nu_k^{(t)}$
        \State Compute per-rank Adam SNR $\mathrm{snr}_k^{(t)}$ using Eq.~\ref{eq:snr}
        \State Normalize $\bar{\nu}_k^{(t)}$ and $\mathrm{snr}_k^{(t)}$ within the module
        \If{$t > T_{\mathrm{warm}}$}
            \State Compute $\log s_k^{(t)}
            \leftarrow
            \alpha\log(1/\tilde{\nu}_k^{(t)})
            + \beta\log\widetilde{\mathrm{snr}}_k^{(t)}$
            \State Clamp $\log s_k^{(t)}$ to $[-\log\kappa,\log\kappa]$
            \State Set $s_k^{(t)} \leftarrow \exp(\log s_k^{(t)})$ and mean-normalize across ranks
        \EndIf
        \State Apply post-step delta scaling:
        $\theta_k^{(t+1)} \leftarrow
        \theta_k^{(t)} +
        s_k^{(t)}(\hat{\theta}_k^{(t+1)}-\theta_k^{(t)})$
    \EndFor
\EndFor
\end{algorithmic}
\end{algorithm*}

\section{Experimental Setups}
\subsection{Datasets}
\label{app:datasets}

We describe the datasets used in our experiments on commonsense reasoning, generative tasks and visual instruction tuning.

\textbf{Commonsense Reasoning.}
\begin{enumerate}
    \item \textbf{Commonsense170K}~\citep{hu2023llm} is a large-scale instruction tuning dataset containing approximately 170K commonsense reasoning examples collected from multiple benchmarks.

    \item \textbf{BoolQ}~\citep{clark2019boolq} (3.2K test set) is a yes/no question answering benchmark built from naturally occurring Google search queries paired with Wikipedia passages.

    \item \textbf{PIQA}~\citep{bisk2020piqa} (3K test set) evaluates physical commonsense reasoning by asking models to select the more plausible solution to an everyday task.

    \item \textbf{SIQA}~\citep{sap2019socialiqa} (1.95K test set) measures commonsense reasoning about social situations and human interactions.

    \item \textbf{HellaSwag}~\citep{zellers2019hellaswag} (10K test set) is a commonsense completion benchmark requiring models to select the most plausible continuation from multiple candidates.

    \item \textbf{WinoGrande}~\citep{sakaguchi2021winogrande} (1.7K test set)  evaluates commonsense reasoning through pronoun resolution problems with ambiguous references.

    \item \textbf{ARC-e} (2.3K test set) and \textbf{ARC-c} (1.1K test set)~\citep{clark2018think} are science question answering benchmarks containing elementary-level and challenge-level multiple-choice questions, respectively.

    \item \textbf{OBQA}~\citep{mihaylov2018can} (500 test set) is an open-book question answering benchmark requiring external commonsense and science knowledge.
\end{enumerate}

\textbf{Natural Language Generation. }
\begin{enumerate}
    \item \textbf{MetaMathQA}~\citep{metamath} is a large-scale mathematical instruction tuning dataset. Following prior work, we use its 100K subset for training in math reasoning experiments.

    \item \textbf{GSM8K}~\citep{gsm8k} (1.3K test set) is a benchmark of grade-school mathematical reasoning problems.

    \item \textbf{Code-Feedback}~\citep{codefeedback} contains code solutions paired with natural language feedback. We use its 100K subset for code generation training.

    \item \textbf{HumanEval}~\citep{humaneval} (164 test set) evaluates code generation through function synthesis problems with unit-test-based execution evaluation.
    
    \item \textbf{WizardLM}~\citep{xu2025wizardlmempoweringlargepretrained} is an instruction tuning dataset generated with Evol-Instruct, which progressively rewrites seed instructions into more complex and diverse training examples. 
    
    \item \textbf{MT-Bench}~\citep{zheng2023judgingllmasajudgemtbenchchatbot} (80 test set) is a multi-turn dialogue benchmark evaluated using LLM-based judgment.

\end{enumerate}

\textbf{Vision-Language Benchmarks.}
\begin{enumerate}
    \item 
    \textbf{LLaVA-1.5 Instruction Tuning Dataset}~\citep{liu2024improvedbaselinesvisualinstruction} contains approximately 665K multimodal instruction-following examples used for visual instruction tuning.
    
    \item \textbf{VQA$^{\text{v2}}$}~\citep{goyal2017makingvvqamatter} (10.7K test set) is a visual question answering benchmark containing open-ended questions about natural images.

    \item \textbf{GQA}~\citep{hudson2019gqanewdatasetrealworld} (12.5K test set) evaluates compositional reasoning and relational understanding in visual question answering.

    \item \textbf{VizWiz}~\citep{gurari2018vizwizgrandchallengeanswering} (8K test set) is a visual question answering benchmark built from images captured by blind users.

    \item \textbf{SQA}~\citep{lu2022learn} (4.2K test set) is a science-oriented visual reasoning benchmark based on diagram understanding.

    \item \textbf{VQA$^{\text{T}}$}~\citep{Singh2019Towards} (5k test set) evaluates visual question answering with text-rich visual inputs.

    \item \textbf{POPE}~\citep{li2023evaluating} (9K test set) evaluates object hallucination in multimodal large language models.

    \item \textbf{MMBench}~\citep{liu2024mmbench} (8.2K test set) is a comprehensive multimodal benchmark covering perception, reasoning and instruction following.
\end{enumerate}

\subsection{Implementation Details}
\label{app:impl_details}

All experiments were conducted on NVIDIA H100 and H200 GPUs, with each finetuning run requiring at most several GPU-hours. We implement all methods using the PEFT library\footnote{https://github.com/huggingface/peft}. We report the hyperparameter settings for commonsense reasoning in Tab.~\ref{tab:hyper_llama}, covering finetuning on LLaMA2-7B\footnote{https://huggingface.co/meta-llama/Llama-2-7b-hf}, LLaMA3-8B\footnote{https://huggingface.co/meta-llama/Meta-Llama-3-8B} and Qwen2.5-7B\footnote{https://huggingface.co/Qwen/Qwen2.5-7B-Instruct-1M}. The shared hyperparameter configuration for dialogue, math reasoning and code generation experiments with LLaMA2-7B is reported in Tab.~\ref{tab:hyper_nlg}. For visual instruction tuning with LLaVA-1.5-7B, we report the corresponding hyperparameter configuration in Tab.~\ref{tab:hyper_visual}. For all settings, we insert LoRA modules into the query, key and value projections of attention layers and the up- and down-projection layers of FFN modules. The log-additive fusion hyperparameters $\alpha$, $\beta$ and $\kappa$ are set to $2.0$, $1.0$ and $3.0$, respectively, with robustness analyses provided in the Appendix. We perform multiple runs with independent seeds and report the average results. 

\paragraph{Baseline tuning protocol.} All baselines share the training configuration of our method as reported in Tab.~\ref{tab:hyper_llama}, Tab.~\ref{tab:hyper_nlg} and Tab.~\ref{tab:hyper_visual}, with method-specific hyperparameters set to their officially recommended values. A further learning-rate sweep on LLaMA2-7B commonsense reasoning (Tab.~\ref{tab:baseline_lr}) confirms that rsLoRA and LoRA+ attain their best accuracy at the learning rate used in our setup. AnLR-LoRA outperforms all baselines at every learning rate in the sweep, including each baseline's own best setting, and the reported gains therefore do not stem from favorable learning-rate choices.

\begin{table}[t]
\centering
\small
\begin{tabular}{lc}
\toprule
Hyperparameters & Value \\
\midrule
Rank $r$ & 32 \\
LoRA Alpha & 64 \\
Dropout & 0.05 \\
Optimizer & AdamW \\
LR  & 1e-4 \\
LR Scheduler & Linear \\
Batch size & 16 \\
Epochs & 3 \\
\bottomrule
\end{tabular}
\caption{Hyperparameter configuration on commonsense reasoning.}
\label{tab:hyper_llama}
\end{table}

\begin{table}[t]
\centering
\small
\begin{tabular}{lc}
\toprule
Hyperparameters & Value \\
\midrule
Rank $r$ & 8 \\
LoRA Alpha & 16 \\
Dropout & 0.05 \\
Optimizer & AdamW \\
LR & 3e-4 \\
LR Scheduler & Linear \\
Batch size & 16 \\
Epochs & 3 \\
\bottomrule
\end{tabular}
\caption{Hyperparameter configuration on dialogue, math reasoning and code generation.}
\label{tab:hyper_nlg}
\end{table}

\begin{table}[t]
\centering
\small
\begin{tabular}{lc}
\toprule
Hyperparameters & LLaVA-1.5-7B \\
\midrule
Rank $r$ & 128 \\
LoRA Alpha & 256 \\
Dropout & 0.05 \\
Optimizer & AdamW \\
LR & 2e-4 \\
LR Scheduler & Linear \\
Batch size & 16 \\
Epochs & 1 \\
\bottomrule
\end{tabular}
\caption{Hyperparameter configuration on visual instruction tuning.}
\label{tab:hyper_visual}
\end{table}

\begin{table}[t]
\centering
\resizebox{0.8\linewidth}{!}{%
\begin{tabular}{lcccc}
\toprule
Method & 5e-5 & 1e-4 & 2e-4 & 3e-4 \\
\midrule
LoRA      & 66.2 & 77.5 & 77.6 & 74.9 \\
rsLoRA    & 49.2 & 79.0 & 61.5 & 63.4 \\
LoRA+     & 66.0 & 79.4 & 77.0 & 74.9 \\
AnLR-LoRA & 79.9 & 81.0 & 80.9 & 78.8 \\
\bottomrule
\end{tabular}%
}
\caption{Average commonsense reasoning accuracy (\%) across varying learning rates on LLaMA2-7B.}
\label{tab:baseline_lr}
\end{table}

\section{Licenses and Usage}
\label{appendix:license}

All datasets and pretrained models used in this work are publicly available and used in accordance with their respective licenses and terms of use, including LLaMA2, LLaMA3, Qwen2.5 and LLaVA-1.5-7B. Our use of these resources is limited to academic research and evaluation purposes.

\section{More Experiments and Analyses}
\subsection{Analysis of Spectral Concentration}
\label{appendix:spectrum}
\paragraph{The full fine-tuning update has a nearly flat spectrum within the rank budget.}
We take full fine-tuning of LLaMA2-7B on Commonsense170K as a reference for the update that a rank-$32$ adapter approximates. 
We measure the effective rank of an update as the exponential of the Shannon entropy of its singular-energy distribution, which equals $k$ when the energy is spread uniformly over $k$ directions~\citep{roy2007effective}.
The best rank-$32$ approximation of the full fine-tuning update has an effective rank of $27.2$ out of $32$, whereas uniform-LR LoRA at rank $32$ learns an update with effective rank $12.7$. This gap bounds how much of the full fine-tuning update a LoRA update can capture, since by the Eckart–Young theorem an update that spreads its energy over $k$ directions captures at most the energy in the top-$k$ singular values. At the learned effective rank of 12.7, this ceiling is $59.4\%$ of what the rank-$32$ budget could capture, compared with $90.5\%$ at the effective rank of the full fine-tuning update. Moreover, the trailing singular directions (indices 25--32) carry $28\%$ as much energy of the full fine-tuning update per direction as the leading ones, yet receive only $2\%$ as much energy from LoRA. These directions therefore hold adaptation that uniform-LR training leaves unused.

\paragraph{Flattening the spectrum directly improves accuracy.}
To isolate the effect of a flatter spectrum on adaptation, we train standard LoRA with an additional loss penalty on the participation ratio of the singular values of $BA$ as a diagnostic intervention. The penalty is minimized when the spectrum is flat and involves no learning-rate intervention. For the LLaMA2-7B experiment on math reasoning, with a penalty coefficient of $10^{-3}$, the performance improves from 59.59 to 60.96 and the effective rank of the learned update reaches 7.22 out of 8. Increasing the coefficient to $10^{-2}$ flattens the spectrum further to an effective rank of 7.89 and raises accuracy to 61.11. Broader use of the rank budget is therefore itself beneficial.

\subsection{Mechanism Analysis of the Two Signals}
\label{appendix:mechanism}
This section examines the training dynamics underlying the two signals of AnLR-LoRA, measuring what velocity and Adam SNR capture in LoRA training and how each term of the multiplier shapes the learned update.

\paragraph{Low-velocity components carry under-exploited learning signal.} We analyze whether slow rank-one components carry meaningful learning signal that a larger learning rate could exploit, using a vanilla LoRA run on commonsense reasoning with LLaMA2-7B. We track the velocity, SNR and update-direction consistency of each rank-one component throughout training. Slow and fast rank-one components show comparable direction consistency, with mean cosine similarities of 0.22 and 0.24, indicating that low velocity does not correspond to an absence of a stable update direction. The difference in update velocity arises from the multiplicative structure of $BA$, under which the rate of change of a rank-one contribution scales with the norms of its factors. Within the slow half of the components, higher SNR is associated with higher direction consistency, and the SNR term distinguishes slow components with stable directions from those with unstable ones. The combined velocity–SNR score predicts the subsequent growth of each component's contribution with a Spearman correlation of 0.69. The components that the multiplier boosts are therefore those that continue to develop under uniform LR but at a slower rate, and redistributing learning rate toward them recovers the under-exploited learning signal.

\paragraph{Effect of each term.}
For the uniform, velocity-only, SNR-only and velocity+SNR configurations of Tab.~\ref{tab:sched_ablation} on LLaMA2-7B, we log per-rank velocity, per-rank SNR and the applied multipliers across all target modules. The velocity term reduces the within-module velocity imbalance, halving the coefficient of variation of per-rank velocity from 0.12 under uniform LR to 0.06. Without the velocity term, the imbalance remains at the level of uniform LR. The SNR term concentrates the redistributed learning rate on reliable components, lowering the share of the total boost received by the bottom-SNR quartile from 37.4\% under velocity-only scaling to 18.3\%, below the proportional share of 25\%. The two terms jointly reduce the fraction of spectrally concentrated modules from 28.1\% to 16.3\%, and neither term alone reaches the accuracy of their combination in Tab.~\ref{tab:sched_ablation}.

\subsection{Robustness to LR Choice}
\label{appendix:lr_sweep}
Tab.~\ref{tab:lr_sweep} provides the full per-task results for the varying LR experiment discussed in Sec.~\ref{sec:exp}. We evaluate six learning rates from $\eta = 5\text{e-}5$ to $5\text{e-}4$ on LLaMA-2-7B. LoRA is highly sensitive to the choice of learning rate. Its average accuracy peaks at $77.6\%$ near $\eta = 2\text{e-}4$ and drops substantially when the learning rate is smaller or larger, reaching $66.2\%$ at $\eta = 5\text{e-}5$ and $61.1\%$ at $\eta = 5\text{e-}4$. Across the evaluated range, LoRA varies by $16.5\%$, and no setting of the global learning rate lifts LoRA above the ceiling of $77.6\%$. AnLR-LoRA improves over LoRA at every learning rate, with gains ranging from $3.3\%$ at LoRA’s best setting to $13.7\%$ at the lower end of the sweep. AnLR-LoRA also outperforms LoRA at LoRA’s optimal learning rate, suggesting a genuine performance gain that cannot be explained by LR tuning alone. The total variation of AnLR-LoRA across evaluated LR range is only $7.7\%$ ($73.3\%$ to $81.0\%$), which is much more robust compared to LoRA. These results demonstrate adaptive anisotropic learning rate at the rank-one level improves both accuracy and robustness across global learning-rate choices.

\begin{table*}[htp]
\centering
\resizebox{\linewidth}{!}{%
\begin{tabular}{cl cccccccc c}
\toprule
\textbf{LR} & \textbf{Method} & \textbf{BoolQ} & \textbf{PIQA} & \textbf{SIQA} & \textbf{HellaSwag} & \textbf{WinoGrande} & \textbf{ARC-e} & \textbf{ARC-c} & \textbf{OBQA} & \textbf{Avg.} \\
\midrule
 \multirow{2}{*}{5e-5}
 & LoRA      & \textbf{71.8} & 82.4 & 65.7 & 51.9 & \textbf{82.2} & 66.9 & 57.9 & 50.6 & 66.2 \\
 & AnLR-LoRA & 70.6 & \textbf{83.3} & \textbf{79.5} & \textbf{89.1} & 82.0 & \textbf{84.7} & \textbf{69.9} & \textbf{80.4} & \textbf{79.9} \\
\midrule
\multirow{2}{*}{1e-4}
 & LoRA      & 72.2 & \textbf{84.6} & 77.4 & 90.1 & 82.6 & 81.2 & 67.4 & 64.8 & 77.5 \\
  & AnLR-LoRA & \textbf{72.5} & 84.1 & \textbf{78.9} & \textbf{91.4} & \textbf{82.8} & \textbf{85.0} & \textbf{71.3} & \textbf{82.4} & \textbf{81.0} \\
\midrule
\multirow{2}{*}{2e-4}
 & LoRA & 69.8 & 79.9 & 79.5 & 83.6 & 82.6 & 79.8 & 64.7 & 81.0 & 77.6 \\
 & AnLR-LoRA & \textbf{72.1} & \textbf{83.4} & \textbf{80.0} & \textbf{91.0} & \textbf{83.7} & \textbf{84.5} & \textbf{70.2} & \textbf{82.2} & \textbf{80.9} \\
\midrule
\multirow{2}{*}{3e-4}
 & LoRA      & 62.9 & 79.5 & 77.9 & 82.3 & 79.7 & 78.8 & 60.5 & 77.4 & 74.9 \\
 & AnLR-LoRA & \textbf{71.6} & \textbf{81.8} & \textbf{78.7} & \textbf{86.8} & \textbf{82.1} & \textbf{81.9} & \textbf{66.5} & \textbf{81.2} & \textbf{78.8} \\
\midrule
\multirow{2}{*}{4e-4}
 & LoRA      & 62.7 & 71.1 & \textbf{69.3} & 63.3 & 70.1 & 73.5 & 58.0 & 66.0 & 66.7 \\
 & AnLR-LoRA & \textbf{69.1} & \textbf{81.6} & 62.8 & \textbf{82.3} & \textbf{79.1} & \textbf{76.4} & \textbf{62.0} & \textbf{75.8} & \textbf{73.6} \\
\midrule
\multirow{2}{*}{5e-4}
 & LoRA      & 62.2 & 74.6 & 60.9 & 38.3 & 68.2 & 74.0 & 53.2 & 57.0 & 61.1 \\
 & AnLR-LoRA & \textbf{66.9} & \textbf{79.5} & \textbf{76.2} & \textbf{78.4} & \textbf{77.2} & \textbf{74.2} & \textbf{59.5} & \textbf{74.8} & \textbf{73.3} \\
\bottomrule
\end{tabular}%
}
\caption{Comparison of LoRA and AnLR-LoRA with varying learning rates on LLaMA2-7B.}
\label{tab:lr_sweep}
\end{table*}

\begin{table}[htp]
\centering
\resizebox{0.9\linewidth}{!}{%
\begin{tabular}{clcc}
\toprule
Rank & Method & \#Params (\%) & Acc. \\
\midrule
\multirow{2}{*}{$r=4$}
  & LoRA      & 0.10 & 58.30 \\
  & AnLR-LoRA & 0.10 & \textbf{60.96} \\
\midrule
\multirow{2}{*}{$r=8$}
  & LoRA      & 0.21 & 59.59 \\
  & AnLR-LoRA & 0.21 & \textbf{61.87} \\
\midrule
\multirow{2}{*}{$r=32$}
  & LoRA      & 0.83 & 58.23 \\
  & AnLR-LoRA & 0.83 & \textbf{61.64} \\
\bottomrule
\end{tabular}
}
\caption{
Ablation study on LoRA rank with LLaMA2-7B on math reasoning task. 
}
\label{tab:ablate_rank}
\end{table}

\begin{table*}[htp]
\centering
\resizebox{0.95\linewidth}{!}{%
\begin{tabular}{l cccccccc c}
\toprule
 \textbf{LR Scaling} & \textbf{BoolQ} & \textbf{PIQA} & \textbf{SIQA} & \textbf{HellaSwag} & \textbf{WinoGrande} & \textbf{ARC-e} & \textbf{ARC-c} & \textbf{OBQA} & \textbf{Avg.} \\
\midrule

 Random          & 62.1 & 73.0 & 39.1 & 57.9 & 79.2 & 45.3 & 42.3 & 38.8 & 54.7 \\
$\|\nabla W \odot W\|$ & 62.2 & 83.9 & \textbf{79.5} & 89.5 & 82.7 & 84.2 & 70.9 & 81.2 & 79.3 \\
 AnLR-LoRA  & \textbf{72.5} & \textbf{84.1} & 78.9 & \textbf{91.4} & \textbf{82.8} & \textbf{85.0} & \textbf{71.3} & \textbf{82.4} & \textbf{81.0} \\
\bottomrule
\end{tabular}%
}
\caption{Ablation on per-rank-one LR scaling strategy.}
\label{tab:lr_scaling_ablation}
\end{table*}

\subsection{Per-rank-one LR Scaling Signal}
Tab.~\ref{tab:lr_scaling_ablation} compares three choices of per-rank-one LR multipliers on commonsense reasoning. Random multipliers collapse the average accuracy to $54.7\%$, $26.3\%$ below AnLR-LoRA, confirming that the multipliers must derive from a meaningful training-time signal rather than from random perturbation. The parameter-sensitivity signal $\|\nabla W \odot W\|$~\citep{adalora} recovers most of the performance ($79.3\%$) and outperforms standard LoRA with uniform-LR ($77.5\%$), indicating that other principled signals can also populate the adaptive anisotropic LR framework. AnLR-LoRA's velocity and Adam SNR fusion achieves the highest average accuracy ($81.0\%$), which is $1.7\%$ above the parameter-sensitivity alternative.

\subsection{Robustness to Rank Choice}
Tab.~\ref{tab:ablate_rank} reports the effect of the LoRA rank on math reasoning. AnLR-LoRA improves over LoRA at every rank tested ($r \in \{4, 8, 32\}$) by $2.3\%$ to $3.4\%$, with no additional parameter cost. AnLR-LoRA at $r =4$ ($60.96\%$) already exceeds LoRA at $r = 32$ ($58.23\%$), consistent with our finding that AnLR-LoRA makes broader use of the nominal rank budget.

\begin{figure}[htp]
  \centering
  \begin{subfigure}{0.48\linewidth}
      \centering
      \includegraphics[width=\linewidth]{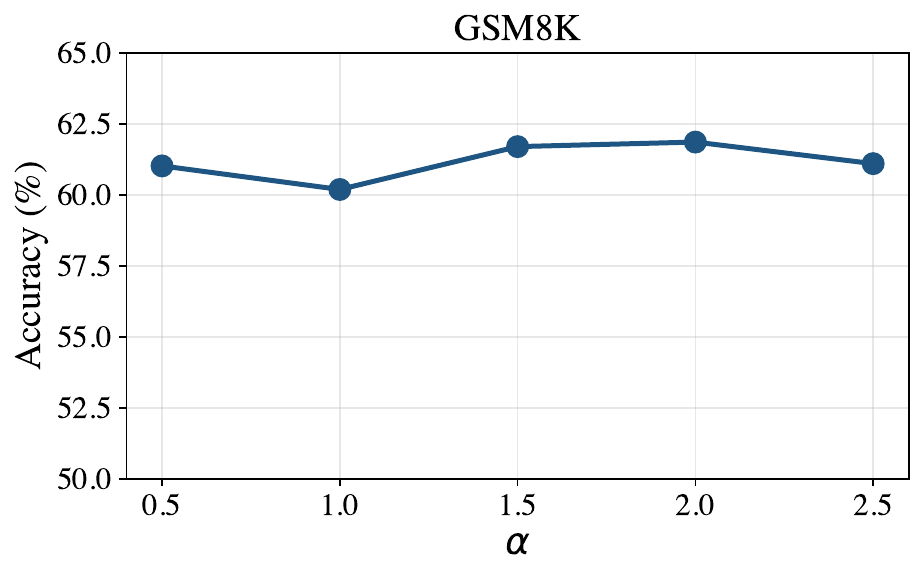}
  \end{subfigure}
  \hfill
  \begin{subfigure}{0.48\linewidth}
      \centering
      \includegraphics[width=\linewidth]{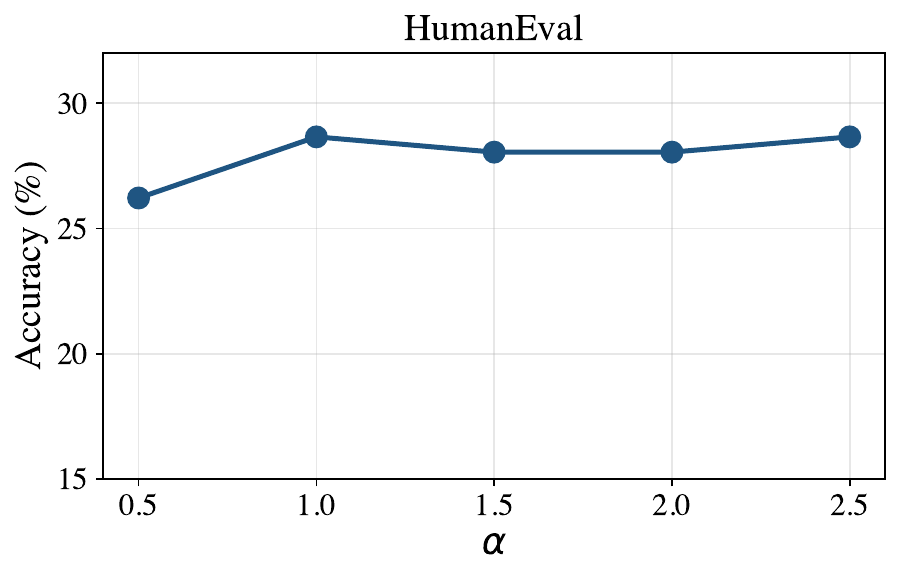}
  \end{subfigure} 
  \caption{Accuracy across the velocity weight $\alpha$ on math reasoning and code generation tasks.}
  \label{fig:alpha_ablation}
\end{figure} 

\begin{figure}[htp]
  \centering
  \begin{subfigure}{0.48\linewidth}
      \centering
      \includegraphics[width=\linewidth]{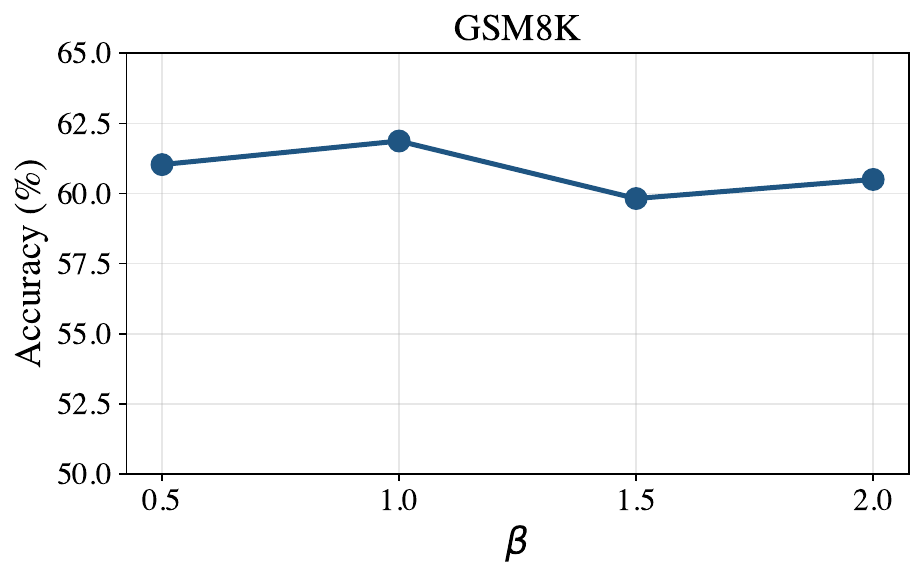}
  \end{subfigure}
  \hfill
  \begin{subfigure}{0.48\linewidth}
      \centering
      \includegraphics[width=\linewidth]{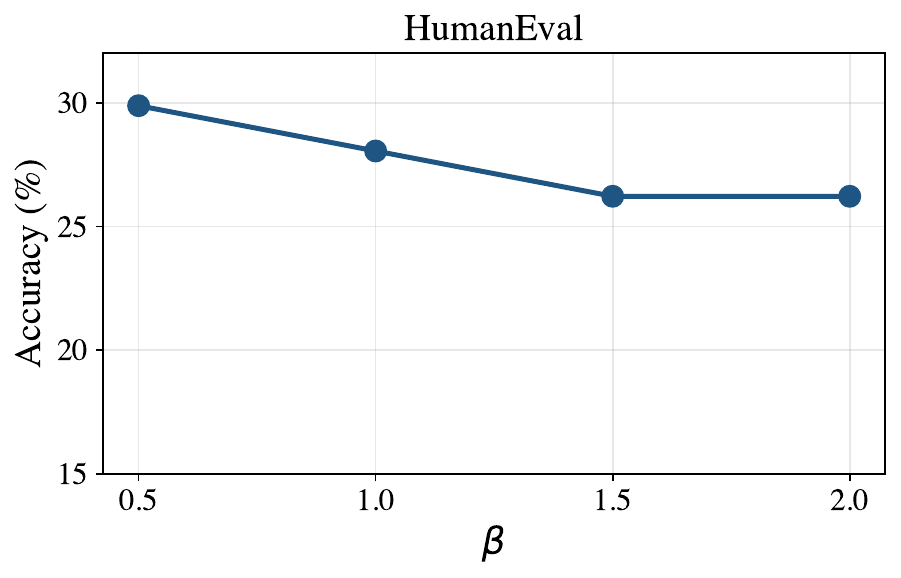}
  \end{subfigure} 
  \caption{Accuracy across the SNR weight $\beta$ on math reasoning and code generation tasks.}
  \label{fig:beta_ablation}
\end{figure} 

\subsection{Robustness to Hyperparameter Choice}
Fig.~\ref{fig:alpha_ablation} and Fig.~\ref{fig:beta_ablation} report AnLR-LoRA's sensitivity to the fusion weights $\alpha$ and $\beta$ on math reasoning and code generation tasks. Across both tasks, performance is relatively stable across the evaluated range of $\alpha \in \{0.5, 1.0, 1.5, 2.0, 2.5\}$ and $\beta \in \{0.5, 1.0, 1.5, 2.0\}$. Fig.~\ref{fig:rho_ablation}, Fig.~\ref{fig:kappa_ablation} and Fig.~\ref{fig:warmup_ablation} further sweep the velocity EMA coefficient $\rho$, the clamp bound $\kappa$ and the warmup length $T_{\mathrm{warm}}$ on commonsense reasoning with LLaMA2-7B and LLaMA3-8B, showing similarly stable results across the evaluated values. AnLR-LoRA is therefore robust to the choice of hyperparameters and does not require careful tuning to achieve strong performance.

\begin{figure}[htp]
  \centering
  \begin{subfigure}{0.48\linewidth}
      \centering
      \includegraphics[width=\linewidth]{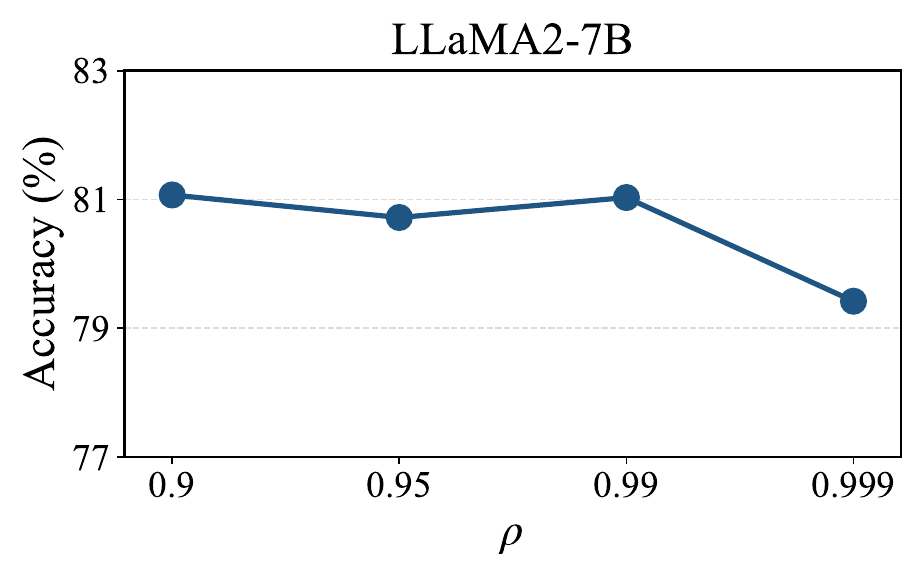}
  \end{subfigure}
  \hfill
  \begin{subfigure}{0.48\linewidth}
      \centering
      \includegraphics[width=\linewidth]{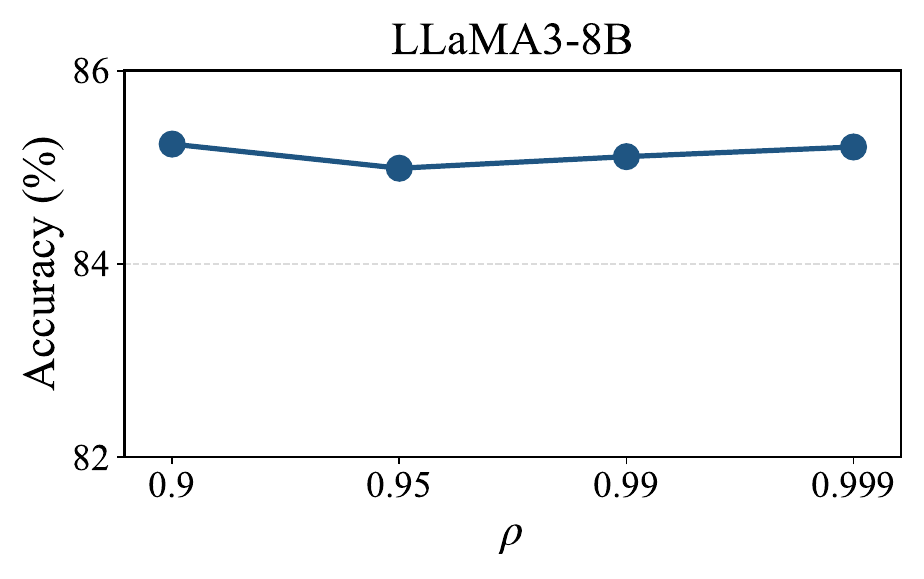}
  \end{subfigure}
  \caption{Average commonsense reasoning accuracy across the velocity EMA coefficient $\rho$.}
  \label{fig:rho_ablation}
\end{figure}

\begin{figure}[h]
  \centering
  \begin{subfigure}{0.48\linewidth}
      \centering
      \includegraphics[width=\linewidth]{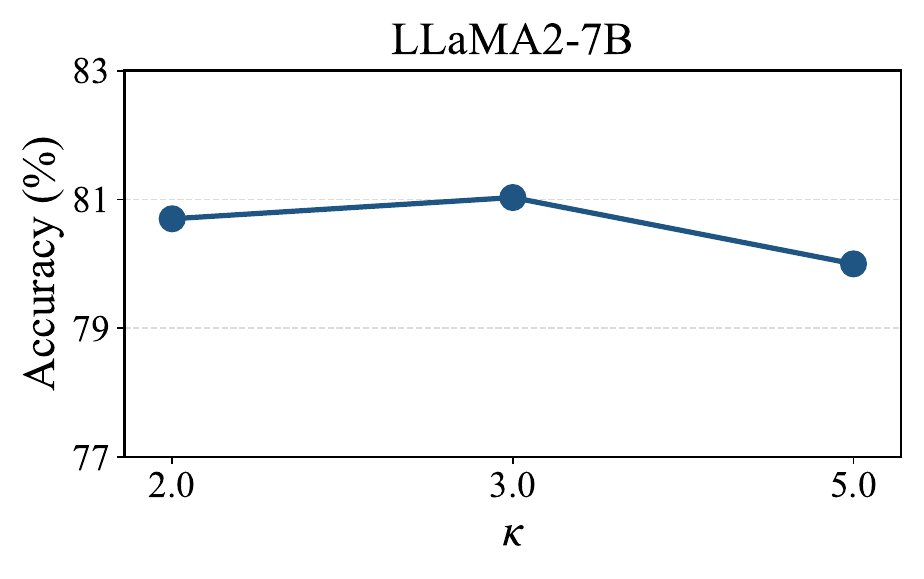}
  \end{subfigure}
  \hfill
  \begin{subfigure}{0.48\linewidth}
      \centering
      \includegraphics[width=\linewidth]{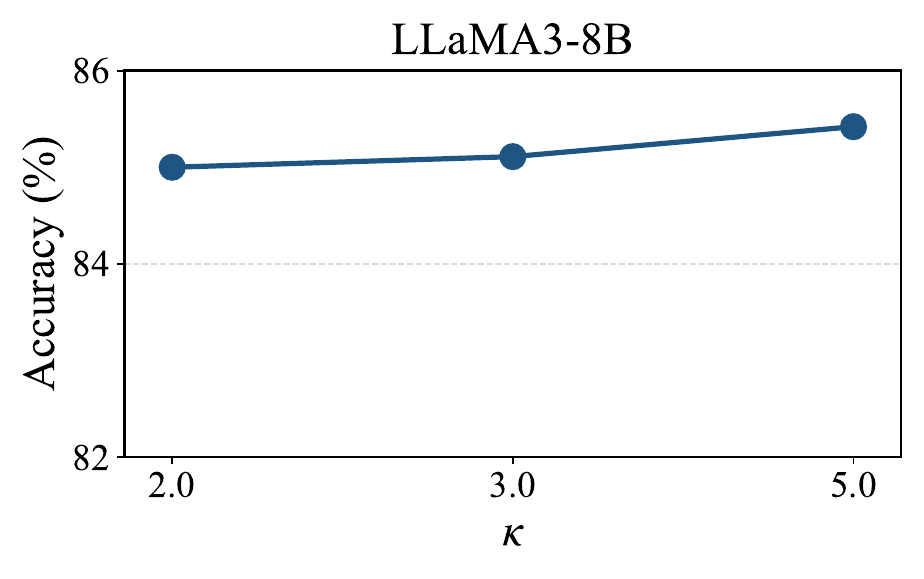}
  \end{subfigure}
  \caption{Average commonsense reasoning accuracy across the clamp bound $\kappa$.}
  \label{fig:kappa_ablation}
\end{figure}

\begin{figure}[h]
  \centering
  \begin{subfigure}{0.48\linewidth}
      \centering
      \includegraphics[width=\linewidth]{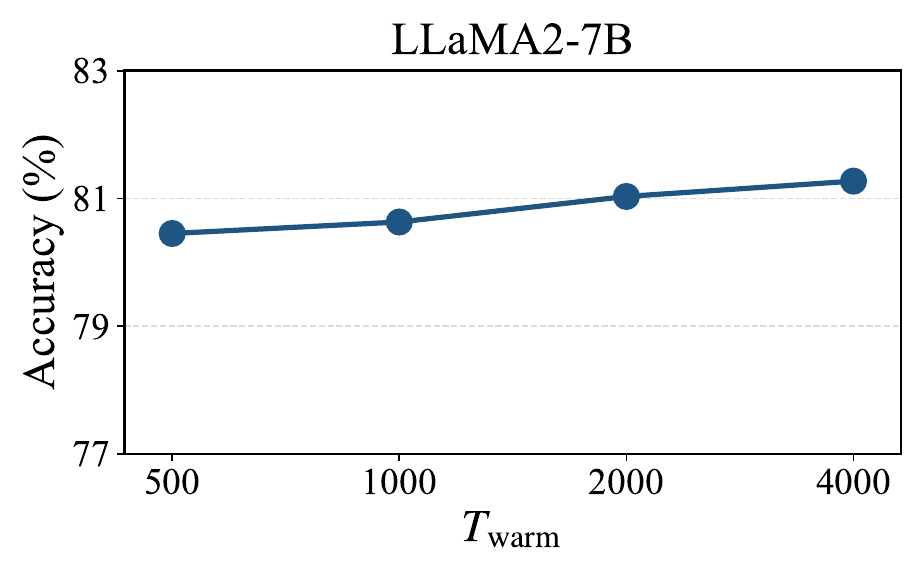}
  \end{subfigure}
  \hfill
  \begin{subfigure}{0.48\linewidth}
      \centering
      \includegraphics[width=\linewidth]{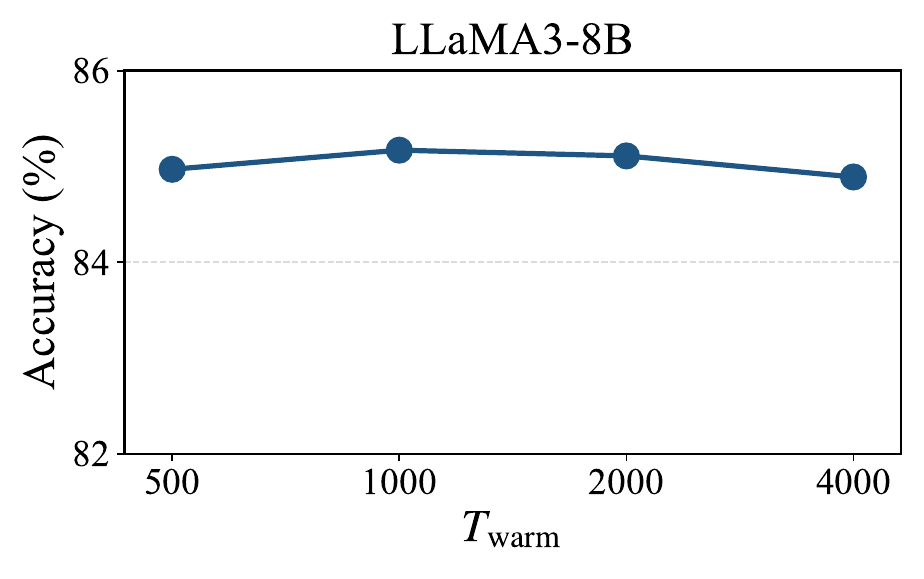}
  \end{subfigure}
  \caption{Average commonsense reasoning accuracy across the warmup length $T_{\mathrm{warm}}$.}
  \label{fig:warmup_ablation}
\end{figure}

\begin{table}[t]
\centering
\resizebox{\linewidth}{!}{%
\begin{tabular}{lccc}
\toprule
\textbf{Method} & \textbf{LLaMA2-7B} & \textbf{LLaMA3-8B} & \textbf{Qwen2.5-7B} \\
\midrule
LoRA      & 77.5 $\pm$ 1.85 & 80.8 $\pm$ 0.43 & 85.7 $\pm$ 0.34 \\
rsLoRA    & 79.0 $\pm$ 1.10 & 81.9 $\pm$ 0.30 & 85.6 $\pm$ 0.11 \\
LoRA+     & 79.4 $\pm$ 0.38 & 84.8 $\pm$ 0.19 & 85.7 $\pm$ 0.38 \\
AnLR-LoRA & \textbf{81.0} $\pm$ 0.44 & \textbf{85.1} $\pm$ 0.14 & \textbf{87.4} $\pm$ 0.13 \\
\bottomrule
\end{tabular}%
}
\caption{Accuracies with standard deviations on commonsense reasoning tasks.}
\label{tab:std1}
\end{table}

\begin{table}[!htbp]
\centering
\resizebox{\linewidth}{!}{%
\begin{tabular}{lccc}
\toprule
\textbf{Method} & \textbf{GSM8K} & \textbf{HumanEval} & \textbf{Visual Instr.\ Tuning} \\
\midrule
LoRA      & 59.59 $\pm$ 0.79 & 24.39 $\pm$ 0.93 & 68.26 $\pm$ 0.13 \\
LoRA+     & 60.27 $\pm$ 0.62 & 24.39 $\pm$ 1.61 & 67.40 $\pm$ 0.14 \\
AnLR-LoRA & \textbf{61.87} $\pm$ 0.53 & \textbf{28.05} $\pm$ 0.93 & \textbf{68.80} $\pm$ 0.10 \\
\bottomrule
\end{tabular}%
}
\caption{Accuracies with standard deviations on natural language generation and visual instruction tuning tasks.}
\label{tab:std2}
\end{table}

\subsection{Standard Deviation Across Seeds}
We conduct all main experiments over three independent runs with different random seeds and report the mean and standard deviation in Tabs.~\ref{tab:std1} and~\ref{tab:std2}. The low variance across seeds indicates that the improvements of AnLR-LoRA over LoRA are robust to the choice of random seed.

\begin{table}[!htp]
\centering
\small
\begin{tabular}{lcc}
\toprule
\textbf{Method} & \textbf{\#Params (\%)} & \textbf{GSM8K} \\
\midrule
LoRA      & 0.16 & 87.49 \\
rsLoRA    & 0.16 & 84.00 \\
LoRA+     & 0.16 & 87.00 \\
AnLR-LoRA & 0.16 & \textbf{88.32} \\
\bottomrule
\end{tabular}
\caption{Evaluation results with Qwen2.5-14B.}
\label{tab:qwen25_14b}
\end{table}

\subsection{Experiments on Qwen2.5-14B}
We further extend the math reasoning evaluation to the larger Qwen2.5-14B backbone (Tab.~\ref{tab:qwen25_14b}). AnLR-LoRA attains the best accuracy, outperforming LoRA, LoRA+ and rsLoRA by 0.83, 1.32 and 4.32 points, respectively. The consistent gain shows that adaptive anisotropic LR remains effective at larger model scales.

\begin{table}[!htp]
\centering
\resizebox{0.95\linewidth}{!}{%
\begin{tabular}{lccc}
\toprule
\textbf{Method} & \textbf{LLaMA2-7B} & \textbf{LLaMA3-8B} & \textbf{Qwen2.5-7B} \\
\midrule
AdaLoRA   & 70.6 & 84.6 & 86.2 \\
DoRA      & 79.6 & 85.0 & 86.8 \\
LoRA-Pro  & 77.6 & 84.9 & 86.2 \\
AnLR-LoRA & \textbf{81.0} & \textbf{85.1} & \textbf{87.4} \\
\bottomrule
\end{tabular}%
}
\caption{Comparison with additional LoRA variants on commonsense reasoning tasks with average accuracy reported.}
\label{tab:more_baselines}
\end{table}

\subsection{Comparison with Additional LoRA Variants}

Tab.~\ref{tab:more_baselines} extends the commonsense reasoning comparison to more LoRA variants, including AdaLoRA~\citep{adalora}, DoRA~\citep{dora} and LoRA-Pro~\citep{lorapro}, trained under the protocol of Tab.~\ref{tab:commonsense} with their official configurations. AdaLoRA adaptively allocates the rank budget across modules by importance-based pruning of singular triplets. DoRA decomposes the adapted weight into a learnable magnitude and a LoRA-updated direction. LoRA-Pro adjusts the gradients to better approximate the full fine-tuning update, requiring per-module matrix inversions and the solution of a Sylvester equation at each optimization step. AnLR-LoRA achieves the best performance on every backbone without introducing additional trainable parameters or expensive matrix operations, relying only on lightweight rescaling of the optimizer updates.

\subsection{Analysis on Memory Overhead of the Parameter Snapshot}
Post-step delta scaling maintains a transient snapshot of the LoRA factors $A$ and $B$ within each optimizer step. The snapshot introduces no trainable parameters, and its size is determined by the adapter rather than by the base model. On LLaMA2-7B, it occupies 56 MB, corresponding to 0.42\% of the base-weight memory. This overhead decreases further for larger backbones, with an estimated 0.20\% at the 70B scale and 0.10\% at the 405B scale. The snapshot can also be eliminated entirely by fusing the rescaling into the optimizer step.

\section{LLM Usage}
\label{app:llm_usage}
AI assistants were used only for language refinement during manuscript preparation. They were not involved in proposing the core method, designing the experiments, or drawing scientific conclusions. Since this work studies parameter-efficient finetuning for large language models, pretrained models including LLaMA2-7B, LLaMA3-8B, Qwen2.5-7B, Qwen2.5-14B and LLaVA-1.5-7B were used as part of the experimental setup and evaluation.

\end{document}